\documentclass[10pt,twocolumn,letterpaper]{article}
\usepackage[moderate]{savetrees}

\usepackage{amsthm}
\usepackage{amsmath}
\usepackage{amssymb}
\usepackage{bm}
\usepackage{mathrsfs}
\usepackage[dvips]{graphicx}

\usepackage{algorithm}
\usepackage{algorithmic}

\usepackage{color}

\usepackage{url}

\usepackage{supertabular}

\newtheorem{thm}{Theorem}

\theoremstyle{definition}

\theoremstyle{remark}

\newtheorem{fact}[thm]{Fact}

\numberwithin{thm}{section}

\DeclareMathAlphabet{\mathsfsl}{OT1}{cmss}{m}{sl}

\renewcommand{\phi}{\varphi}

\newcommand{\R}{\mathbb{R}}

\newcommand{\bX}{\boldsymbol{X}}

\newcommand{\bP}{\boldsymbol{P}}

\newcommand{\bD}{\boldsymbol{D}}

\newcommand{\bV}{\boldsymbol{V}}

\def\E{\mathbb{E}}

\def\bS{\boldsymbol{S}}
\def\b0{\mathbf{0}}

\def\bP{\boldsymbol{P}}

\def\bY{\boldsymbol{Y}}

\def\nnz{\mathsf{nnz}}

\usepackage{3dv}
\usepackage{times}
\usepackage{epsfig}
\usepackage{graphicx}
\usepackage{amsmath}
\usepackage{amssymb}
\usepackage{amsthm}
\usepackage{kantlipsum}
\usepackage{multicol}
\usepackage{multirow}

\usepackage{comment}
\usepackage[linesnumbered,ruled,vlined,algo2e]{algorithm2e}

\usepackage[pagebackref=true,breaklinks=true,colorlinks,bookmarks=false]{hyperref}

\threedvfinalcopy 

\def\threedvPaperID{410} 

\SetKwInput{KwInput}{Input}                
\SetKwInput{KwOutput}{Output}   

\ifthreedvfinal\fi

\newtheorem{lemma}{Lemma}

\newtheorem{proposition}{Proposition}

\usepackage[dvipsnames]{xcolor}

\begin{document}

\title{Scalable Cluster-Consistency Statistics for Robust Multi-Object Matching}

\makeatletter
\newcommand{\printfnsymbol}[1]{%
  \textsuperscript{\@fnsymbol{#1}}%
}

\author{
  Yunpeng Shi\printfnsymbol{1}\hspace{2cm} Shaohan Li\printfnsymbol{2}\hspace{2cm} Tyler Maunu\printfnsymbol{3}\hspace{2cm} Gilad Lerman\printfnsymbol{2}\\
  \printfnsymbol{1}Program in Applied and Computational Mathematics, Princeton University\\
  \printfnsymbol{2}School of Mathematics, University of Minnesota\\
  \printfnsymbol{3}Department of Mathematics, Brandeis University\\
  }

\maketitle
\thispagestyle{empty}


\begin{abstract}

We develop new statistics for robustly filtering corrupted keypoint matches in the structure from motion pipeline. The statistics are based on consistency constraints that arise within the clustered structure of the graph of keypoint matches. The statistics are designed to give smaller values to corrupted matches and than uncorrupted matches. These new statistics are combined with an iterative reweighting scheme to filter keypoints, which can then be fed into any standard structure from motion pipeline. This filtering method can be efficiently implemented and scaled to massive datasets as it only requires sparse matrix multiplication. We demonstrate the efficacy of this method on synthetic and real structure from motion datasets and show that it achieves state-of-the-art accuracy and speed in these tasks. Our code is released at \url{https://github.com/yunpeng-shi/FCC}.

\end{abstract}
\section{Introduction}
\label{sec:intro}


The problem of matching multiple objects, namely, multi-object matching, arises in different computer vision applications, such as 3D shape matching \cite{huang2012optimization}
and structure from motion (SfM) \cite{sfmsurvey_2017}. In SfM  and other 3D reconstruction problems, each object is an image that contains a set of keypoints that are generated and matched using various automatic procedures. In practice, due to differing viewpoints, every pair of images only contains a partial set of shared keypoints among all possible keypoints, and the estimated matches are typically corrupted.
Common sources of corruption of the estimated keypoint matches are scene occlusion, change of illumination, viewing distance and perspective, repetitive patterns and ambiguous symmetry \cite{1dsfm14}.

Formally, the general problem of multi-object matching in SfM aims to assign absolute matchings between keypoints and their corresponding 3D scene points. However, in SfM and 3D reconstruction problems, one only needs to estimate the fundamental matrices of each image given the corrupted partial matches \cite{sfmsurvey_2017}. Therefore, for such applications, it is sufficient to estimate  only a subset of the ground-truth keypoint  matches without estimating all of the absolute matches, since it only takes 7 matches to specify the fundamental matrix~\cite{multiviewbook}. 

The common matching procedure in SfM directly compares SIFT \cite{sift04} feature descriptors and refines them by  bidirectional matching \cite{multiviewbook}. Typically, some classical methods are used to make the SfM pipeline robust to keypoint mismatches. For example, it is common to use either the least median of squares or RANSAC for fundamental matrix estimation. Bundle adjustment, the final step in the SfM pipeline, can also be used to correct keypoint matches. 

Different methods have been proposed for solving the multi-object matching in its broad sense, but they have been either unscalable or practically ineffective for SfM. 
Nevertheless, it is interesting to note mathematical ideas that were used before. 
One basic mathematical framework that applies to the matching of only two objects is that of the quadratic assignment problem \cite{loiola2007survey}. Another mathematical framework for multi-object matching with complete matching (and thus does not apply to partial matching) is Permutation Synchronization (PS) \cite{deepti}. These two formulations are not directly applicable to SfM and are hard to compute~\cite{sahni1976pcomplete,Z2NP,deepti}. 
A more natural formulation for the matching problem in SfM is Partial Permutation Synchronization (PPS) \cite{chen_partial}. It aims to address the setting of partial matching,
where some keypoints in one image may not match those in another image. 
Unlike permutation synchronization, PPS is not a special case of the general problem of group synchronization \cite{cemp,AMP_compact}, and thus there are no clean and universal approaches to address it. Its existing solutions \cite{chen_partial, MatchALS} are computationally demanding, so they may not be effectively used within the SfM pipeline. 

Some recent attempts have tried to leverage the consistency structure of the graph of keypoints (which is different from the common graph used for the PPS problem, whose nodes represent images)~\cite{hu2018distributable,serlin2020distributed}. In particular, \cite{hu2018distributable} uses such ideas for a distributed implementation of \cite{MatchALS}; however, it is not scalable when run on a single machine (it is typically only 10 times faster than  \cite{MatchALS}).
Furthermore, \cite{serlin2020distributed} utilizes QuickMatch~\cite{tron2017fast} on the graph of keypoints to robustly find keypoint matches. However,~\cite{tron2017fast} clusters the keypoints through a for loop over all keypoint matches; thus matching errors may accumulate along this sequential and heuristic procedure. Both \cite{serlin2020distributed} and \cite{tron2017fast} also require  additional keypoint features, so they are not standard PPS algorithms that are purely based on cycle-consistency information in the keypoint graph.
Despite these original and innovative directions, we believe there is still more room to quantitatively explore the structure of this graph and accelerate the standard PPS algorithms to handle large-scale SfM matching data.

The goal of this work is to propose a rather simple and fast method to solve a weaker formulation of the  partial permutation synchronization problem under high corruption. The strategies we propose do not seek to capture all correct keypoint matches but rather find a good, consistent subset of them. Specifically, we seek the intersection of the good matches and the given matches. Most importantly, this method can be effectively used within the SfM pipeline, since one does not need all keypoint matches to estimate downstream quantities, like the fundamental matrix. Our method is motivated by some mathematical insights that have not been fully leveraged in the existing literature.

\subsection{Previous Works}

Several methods were proposed for solving very special instances of multi-object matching by permutation synchronization  \cite{deepti,PPM_vahan,irgcl}; these instances are inapplicable to SfM. 
PPS is more relevant for the general setting of multi-object matching, although it is applied to SfM in a cumbersome way: The estimated absolute partial permutations are used to estimate the ground-truth relative partial permutations and consequently the ground-truth keypoint matches.
In order to solve PPS, several works extended convex relaxation methods for permutation synchronization by relaxing permutation matrices to doubly stochastic matrices 
\cite{chen_partial,MatchALS,MatchADMM_RTR}. However, these methods are often inaccurate, unrobust to corruption, slow and not scalable for practical instances of SfM.
Nonconvex methods for PPS 
include \cite{MatchEig,Nonneg_factor,ConsistentFeature, HIPPI}. 
The spectral method \cite{deepti} and the projected power method (PPM) \cite{PPM_vahan,Chen_PPM} for permutation synchronization can be extended to PPS, but they are not sufficiently accurate and scalable and also require an estimate of the unknown number of total unique keypoints.
Other recent works begin to examine consistency constraints of the underlying keypoint graph~\cite{hu2018distributable,serlin2020distributed} and we discussed them above. Among the aforementioned PPS algorithms, the two fastest ones are  \cite{MatchEig} and \cite{deepti}. However, these and other spectral-based PS/PPS methods \cite{Chen_PPM, PPM_vahan} require computation of the top $m$ eigenvectors, where $m$ is the number of unique keypoints, which is also called the universe size. These eigenvectors form a dense matrix of size $N \times m$ where $N$ is the total number of keypoints in all images. For large SfM problems, $N$ and $m$ can be of order $10^6$ and $10^5$, respectively, so the memory requirement is $>100$ GB memory 
and cannot be addressed by a common personal computer. Thus, in order to handle the large-scale SfM data, it is crucial to only involve sparse matrix operations in PPS algorithms, which has not been considered in previous works.  

\subsection{This Work}

Here we summarize the main contributions of this work.

\begin{itemize}
    \item We propose new path counting statistics motivated by the cycle consistency structure of the multi-object matching problem formulation. These statistics are designed to yield separated values for good and bad keypoint matches, so that a hard thresholding method can be easily applied. Most notably, we propose a novel way to incorporate cross-keypoint paths that yields better separation between inliers and outlier values. This is the first methodology developed to be robust at the level of keypoint matches rather than at the level of partial permutations.
    \item We demonstrate how to efficiently compute the statistics for massive datasets in a completely decentralized way. The method involves sparse matrix multiplications and can be efficiently parallelized. The time and space complexity of our method is significantly lower than other methods given sparse initial matches.   We further propose a novel iterative reweighting procedure to refine our path-counting statistics.
    \item We propose a novel synthetic model of SfM data that more realistically mirrors real scenarios while allowing control of parameters. Our method is competitive even though it does not require the number of keypoints $m$ as input, unlike other methods.
    \item Our method achieves state-of-the-art performance on various real datasets in both accuracy and speed. The method improves the estimation of camera location and rotation when applied to the city-scale Photo Tourism database of \cite{photo_tourism, 1dsfm14}.
\end{itemize}

\subsection{Structure of the Rest of the Paper}
Section \ref{sec:problem_setup} provides a mathematical setup and, in particular, reviews the PPS problem and a broader setting that we aim to address, where we do not solve for the absolute partial permutations. Section \ref{sec:novel_stat} describes new statistics for removing bad keypoint matches and a practical algorithm that applies them. We motivate these statistics by explaining the geometric structure and, more specifically, cycle-consistency of the underlying graph. 
Section \ref{sec:numeric} gives experiments on synthetic and real data that demonstrate the utility and efficiency of these statistics. Finally, \S\ref{sec:conclusion} concludes this work, while discussing open directions.

\section{Problem Setup}
\label{sec:problem_setup}

We assume $n$ images $I_1$, $\ldots$, $I_n$ of a 3D scene, $m$ scene points, and that an algorithm has identified $m_i \leq m$ keypoints in each image $I_i$, $i \in [n]$ that aim to describe a subset of the $m$ scene points. Let $\bX_i^{*} \in \mathbb{R}^{m_i \times m}$ describe the ground-truth matches between scene points and keypoints in image $I_i$. More precisely, its $kl$-th entry is 1 if scene point $l$ corresponds to image keypoint $k$ in image $I_i$, and 0 otherwise. Note that $\bX_i^{*}$ is a partial permutation matrix, that is, it is binary with at most one nonzero element at each row and column. Here in the rest of our notation we use the $*$ superscript to designate ground-truth information, which is unknown to the user.

Given images $I_i$ and $I_j$, we denote the partial permutation that matches keypoints between these images by $\bX_{ij}$. Its $kl$th element is 1 if keypoint $k$ in image $i$ corresponds to keypoint $l$ in image $j$ and 0 otherwise. We think of it as an estimate of the ground-truth partial information $\bX^*_{ij} =  \bX_i^{*} \bX_j^{*\top}$. 
We denote by $N$ the total number of keypoints across all images and form a block matrix for the total keypoint matching $\bX =
(\bX_{ij})_{i,j=1,\dots n} \in \{0,1\}^{N \times N}$. Similarly, we denote $\bX^* =
(\bX^*_{ij})_{i,j=1,\dots n} \in \{0,1\}^{N \times N}$. 
In the input to our problem, two keypoints in the same image are never connected, and thus we set the block diagonal regions of $\bX^*$ and $\bX$ to be $0$.

The PPS formulation for multi-object matching asks to estimate the ground-truth absolute partial permutations $\{\bX_i^{*}\}_{i=1}^n$ given $\bX$, or equivalently, finding the ground-truth relative partial permutation $\bX^*$ from $\bX$. However, for large and sparse $\bX$, the corresponding ground truth $\bX^*$ can be much denser than $\bX$, which makes standard PPS algorithms slow and memory-demanding. Moreover, for SfM, the overly dense matches may largely increase the computational burden of robust algorithms for fundamental matrix estimation. Thus, it is sufficient and in fact natural to try to use the good matches that already exist in the given sparse matches $\bX$. For this reason, our method aims to solve for the intersection between $\bX^*$ and $\bX$. Namely, we seek the sparse matrix of good relative matches $\bX_g: = \bX^*\odot \bX$, where $\odot$ is the elementwise (Hadamard) product.

We form a graph $G = ([N], E)$ whose nodes in
$[N]$ ($[N] = \{1,\ldots,N\}$) index the set of all keypoints in all images and whose edges represent matches between these keypoints. Such a graph was used before in~\cite{hu2018distributable,serlin2020distributed,tron2017fast}.
The matrix $\bX$ is the adjacency matrix for this graph, and we note that this graph is different from the common graph for PPS \cite{chen_partial, Huang13, MatchALS}, whose nodes correspond to images. 

Our corruption model assumes within $G$ good (correct) and bad (incorrect) keypoint matches. We thus partition the set of edges into two parts $E = E_g \cup E_b$: $E_g$ denotes the good edges and $E_b$ denotes the bad edges. The corresponding good and bad graphs are $G_g([N], E_g)$ and $G_b([N], E_b)$, respectively. 
Let $\bX_g$ and $\bX_b$ be the adjacency matrices of $G_g([N],E_g)$ and $G_b([N],E_b)$ with blocks $\{\bX_{ij,g}\}_{i,j=1}^n$ and $\{\bX_{ij,b}\}_{i,j=1}^n$, respectively, so that $\bX = \bX_g + \bX_b$ and similarly $\bX_{ij} = \bX_{ij,g} + \bX_{ij, b}$
for $i$, $j \in [n]$.  As above, the block diagonal regions of $\bX_g$ and $\bX_b$ are always $0$.

One can view this model as elementwise-corruption of $\bX^*$, where each $ij \in E$ $\bX_{ij}$ is potentially corrupted, instead of the inlier-outlier corruption model \cite{chen_partial} that assumes corruption of $\bX^*$ at the block level. We are not aware of any previous PPS work that focus on this elementwise-corruption model.  
In view of this model and the above discussion, our problem is the estimation of $\bX_{g}$ given the measurement matrix $\bX$. In simple words, we seek to detect good keypoint matches within $\bX$. 
Examples that further demonstrate this setting appear in the supplementary material.
 
In the following, we will develop novel statistics to filter bad edges in $G$. 
To formalize our ideas, we use a few different graphs, we summarize them below in Table~\ref{tab:graphref}, even though the last two are defined later. 

\begin{table}[h]
\centering
\resizebox{0.9\columnwidth}{!}{
\begin{tabular}{ c | c | c }
 \bf Graph & \bf Definition & \bf Adj.  \\ \hline
 $G_g([N], E_g)$ & Good keypoint matches & $\bX_g$ \\ \hline
 $G_b([N], E_b)$ & Bad keypoint matches & $\bX_b$ \\ \hline
 $G([N], E)$ & Observed keypoint matches & $\bX$\\ \hline
 $G_D([N], E_D)$ & Within image matches (see \eqref{eq:def:D}) & $\bD$    \\ \hline
 $\hat{G}^*([N], \hat{E}^*)$ & Minimal cycle-consistent   & $\hat{\bX}^*$ \\ 
 & graph containing $G_g$ &
\end{tabular}
}
\caption{Graphs used throughout the paper. Adj. is an abbreviation for adjacency matrix. All graphs are defined on the same set of nodes $[N]$, which indexes the set of keypoints across all images.}\label{tab:graphref}
\end{table}

\section{Novel Statistics for Removing Bad Matches}
\label{sec:novel_stat}

Our method aims to remove bad keypoint matches through novel statistics not yet exploited in the literature, and we accomplish this by examining the structure of the graph $G$ described in \S\ref{sec:problem_setup}. In \S\ref{subsec:cycleconsist}, we describe the notion of cycle consistency and how it fits with our graph $G$. Then, in Sections~\ref{subsec:withinclust} and~\ref{subsec:crossclust}, we formulate the novel statistics, $\bS_1$ and $\bS_2$, that take advantage of two different consistency constraints. In \S\ref{subsec:combined}, we show how we combine these statistics.
The supplementary material illustrates the usefulness of the proposed statistics and their combination for a simple motivating example. Finally, \S\ref{sec:complexity} discusses the computation of these statistics in practice.

\subsection{Cycle Consistency}
\label{subsec:cycleconsist}

We utilize the notion of \emph{cycle consistency} to filter out bad keypoint matches. We cannot use the common notion of cycle consistency in partial permutation synchronization \cite{chen_partial}, since we consider a graph whose nodes represent  keypoints instead of images. As we discuss below, we find it more convenient to define a \emph{cycle-consistent graph} instead of a \emph{consistent cycle}. This notion of cycle consistency was explored before in works such as~\cite{hu2018distributable,serlin2020distributed}.

We say that a graph $G'(V',E')$, where $V' \subset \mathbb{N}$, is cycle-consistent if, whenever $i,j,k \in V'$ and $ik,kj\in E'$, then $ij\in E'$. This definition uses a 3-cycle $\{ij,jk,ki\}$, but one can note that the same property holds for higher-order cycles. We verify this claim for 4-cycles, where the extension to higher-order cycles easily follows by induction. Given $i,j,k,l \in V'$ and assuming that  $ij$, $jk$, $kl \in E'$, then the original definition implies that $ik \in E'$, and since $kl \in E'$, then also $li \in E'$, that is, the 4-cycle $\{ij,jk,kl,li\}$ is in $E'$. We easily note that this definition and its extension to higher-order cycles immediately
imply that cycle-consistent graphs are dense in the following sense:
\begin{proposition}
The connected components of any cycle-consistent graph are complete subgraphs.
\end{proposition}

For any graph of good keypoint matches, $G_g([N], E_g)$, one may extend the set $E_g$ and complete its missing edges in each connected component. This naturally results in the smallest cycle-consistent graph $\hat{G}^*([N], \hat{E}^*)$ that contains $G_g$. We denote its adjacency matrix by $\hat{\bX}^*$ (in \S\ref{subsec:withinclust}, we clarify when $\hat{\bX}^*={\bX}^*$). This graph contains the  complete information for solving the PPS problem. However, to solve our problem it is sufficient for us to try to only recover its subgraph $G_g([N], E_g)$. 

\subsection{$\bS_1$ and Within-Cluster Consistency}
\label{subsec:withinclust}

Our first statistic $\bS_1$ for distinguishing between inlier and outlier matches uses ``within-cluster consistency''. We first define this statistic and then motivate it, while explaining the notion of within-cluster consistency. 

Fix a small integer $q\geq 2$, where we later use the default value of 4 in our algorithm. The within-cluster statistic is then defined as
\begin{equation}\label{eq:s1}
    \bS_1 = \bX^q.
\end{equation} 
For $ij \in E$, $\bS_1(i,j)$ counts the number of paths of length $q$ connecting nodes $i$, $j \in [N]$ in $G([N],E)$.
As we explain below, we generally expect that
\begin{equation}
\label{eq:cond_eq_s1}
\bS_1(i,j) \geq \bS_1(k,l) \ \forall 
ij\in E_g, \ kl\in E_b, \text{ when } \bX\approx \bX_g.
\end{equation}
Indeed, good edges are contained in dense subgraphs in the ideal case while bad edges must straddle two dense subgraphs that may not have many connections between them. 

The following fact illuminates the above idea: \begin{fact}
\label{fact_S1}
    The good subgraph $G_g = ([N], E_g)$ and its minimal cycle-consistent extension $\hat{G}^*([N], \hat{E}^*)$ have $m'$ connected components, where $m' \geq m$. Moreover, the rank of $\hat \bX^*$ is $m'$. 
\end{fact}
This fact follows from the observation that keypoints of different scene points are not connected in $G_g$. This fact further 
implies that a good keypoint match, $ij \in  E_g$, should be within a cluster (a dense subgraph). Therefore there should be many short paths connecting $ij$. On the contrary, a bad match that corresponds to a bad edge $ij$ should be between two clusters, and therefore there should be fewer short paths connecting $ij$. 

Ideally, $m=m'$, in which case $\hat \bX^* = \bX^*$, but in practice $m'$ might be larger as some disconnected subclusters may occur due to the sparsity of the observed matches. 
Fact \ref{fact_S1} suggests a stochastic block model for the underlying structure of $\bX$ that can be revealed by spectral methods \cite{deepti, chen_partial}. These methods often assume a relatively dense $G_g$ so that $m'=m$. Then, by finding a rank $m$ approximation of $\bX$, one can estimate $\hat{\bX}^*$. However, $m$ is unknown in practice and the observed $\bX$ is not exactly low rank due to corruption and sparsity in the observation. Moreover, in large-scale datasets, complete recovery of $\hat{\bX}^*$ is unnecessary and requires large amounts of memory and computational time. 

\subsection{$\bS_2$ and Cross-Cluster Consistency}
\label{subsec:crossclust}

We define our second statistic $\bS_2$ and then motivate it in view of what we call cross-cluster consistency. 
We arbitrarily fix two nonnegative integers $r,s$ such that $r+s=q$ (where $q$ was fixed in \S\ref{subsec:withinclust}, and our default values are $r=s=2$). Let $ \bD$ be a block diagonal matrix with the same block sizes as $\bX$, whose diagonal blocks are
\begin{equation}
\label{eq:def:D}
\bD_{ii} = \boldsymbol{1}_{m_i} \boldsymbol{1}_{m_i}^T - \boldsymbol I_{m_i},
\text{ for } i\in [n]
\end{equation}
where $\boldsymbol{1}_{m_i}$ denotes a column vector of ones in $\R^{m_i}$ and $\boldsymbol I_{m_i}$ denotes the $m_i \times m_i$ identity matrix.  Let $G_D$ denote the graph whose adjacency matrix is $\bD$. This graph connects all pairs of distinct keypoints that belong to the same image, which is in contrast to the graph $G$ connects points across different images.
We define our second statistic as 
\begin{equation}\label{eq:s2}
    \bS_2=\bX^r\bD\bX^s.
\end{equation}

Note that for $ij \in E$, $\bS_2(i,j)$ is the number of paths of length $q+1$ connecting nodes $i$ and $j$ composed of an $r$-length path in $G([N],E)$, then an edge in $E_D$ that connects nodes in the same image, and then an $s$-length path in $G([N],E)$. 
Due to this interpretation and the fact stated below we expect that 
\begin{equation}
\label{eq:cond_eq_s2}
\bS_2(i,j) = 0 \leq \bS_2(k,l)
 \ \forall ij\in E_g, \ kl\in E_b,
  \text{ when } \bX\approx \bX_g.
\end{equation}

\begin{fact}\label{fact:cross}
    The good subgraph $G_g$ and its dense version $\hat{G}^*$ do not contain any path that connects two keypoints in the same image. 
\end{fact}
This fact is obvious as such a path would require either a bad keypoint match or a within image match.
One thus expects that for $ij \in E_g$ $\bS_2(i,j)=0$ as otherwise there is a path with an  edge in $G_D$ that connects between two disconnected subgraphs of $G_g$. 
To the best of our knowledge, this fact has not yet been fully utilized and this turns out to be key for our method. 

To quantify the above idea more precisely, we give the following proposition, whose proof is in the supplemental material.
\begin{proposition}\label{prop:D}
If $G_{\text{sub}}([N],E_{\text{sub}})$ is a subgraph of $\hat{G}^*([N], \hat{E}^*)$ with adjacency matrix $\bY$ and $\bS_2^{\text{sub}}:=\bY^r\bD\bY^s$, then 
\begin{align}\label{eq:constr}
    \bS_2^{\text{sub}}\odot\hat{\bX}^* =\boldsymbol 0.
\end{align}
\end{proposition}

Proposition \ref{prop:D} implies that  $\bS_2^{\text{sub}}(i,j)$$>0$ is a sufficient condition for $\hat{\bX}^*(i,j)=0$, that is, for $ij$ being a bad edge. The next proposition shows that under some assumptions, it is also a necessary condition. Again, its proof can be found in the supplemental material.
\begin{proposition}\label{prop:D2}
Let $G_{\text{sub}}$, $\bY$ and $\bS_2^{\text{sub}}$ be defined as in Proposition \ref{prop:D}. Assume that $G_{\text{sub}}$ contains the same number of connected components as $\hat{G}^*$ and that, for any two components of $G_{\text{sub}}$, there exists an image that contains at least one node in each component. Then for any $(i,j)$ in the off-diagonal blocks of $\hat{\bX}^*$, $\hat{\bX}^*(i,j)=1$ if and only if $\bS_2^{\text{sub}}(i,j)=0$.
\end{proposition}

The additional assumption of Proposition \ref{prop:D2}  holds when 
each pair of 3D points are contained in at least one image, in which case the $\bS_2$ should be very helpful. In cases where the assumption is not satisfied, in particular when two 3D points at opposing ends of a 3D structure cannot be viewed by a single camera, then one may use the within cluster information of $\bS_1$. This motivates the combination of the two statistics seen in the next section.

\subsection{Filtering by Combining $\bS_1$ and $\bS_2$}
\label{subsec:combined}


We construct our combined statistic as 
\begin{equation}
\bS=(\bS_1\oslash(\bS_1+\bS_2))\odot \bX,    
\end{equation}
where $\oslash$ denotes elementwise division.
That is, for any $ij \in E$, the statistic is the ratio between $\bS_1$ and the sum of the two statistics, whereas for $ij \notin E$, it assigns zero values.

We generally use $r=s$ due to symmetry, so that paths from $i$ to $j$ and $j$ to $i$ are similarly treated. We choose $q=r+s$ so that $\bS_1$ and $\bS_2$ have comparable scales, as the number of steps within $G$ is the same. Also, our combined statistic is nicely scaled between 0 and 1. As we show later in \eqref{eq:stat} and \eqref{eq:compS1S2}, $S_1$ and $S_2$ have similar forms, and consequently the sum of $\bS_1(i,j)$ and $\bS_2(i,j)$ can be efficiently vectorized. In practice, we recommend $r=2$ (or equivalently $q=4$), and parameter-tuning is not needed. This choice corresponds to walks of length $4$, which covers simple paths of length $2$ and $4$ (thus it also covers the paths of $r=1$).   Choosing higher $r$ (longer paths) may increase the computational complexity as $\bX^r$ can be dense for large $r$. Furthermore, in the next paragraph we introduce an alternating improvement strategy that allows ``message passing" among distant edges, thus longer paths are not needed. Similar strategy is validated in \cite{cemp} for a different problem.

In view of \eqref{eq:cond_eq_s1} and \eqref{eq:cond_eq_s2},
we expect that $\bS(i,j)$ is small for bad edges. In particular, $\bS(i,j)\in [0,1]$ can in some sense be interpreted as a ``probability" that $ij\in E_g$. By replacing $\bX$ with $\bS$ in the formulas of $\bS_1$ and $\bS_2$, our original path counting procedure becomes a weighted path counting, where the weights focus on the clean paths. This observation motivates an iterative procedure, where the path weights and $\bS$ alternatingly improve each other. In this procedure, the initial input is $\bX$ and then the input is $\bS$ obtained at the previous iteration. At the last iteration (or without any iteration) one can then filter
the $\bS$-scores above a certain threshold and identify the edges whose final scores are nonzero as good ones. One can also threshold at each iteration. For completeness, Algorithm~\ref{alg:fcc} describes this iterative procedure, which we refer to as Filtering by Cluster Consistency (FCC). It uses the notation $\mathbf{1}(\bS > \tau_t)$ for an $N \times N$ binary matrix whose elements are $1$ whenever $\bS_{ij} > \tau_t$.

\begin{algorithm2e}[]
\DontPrintSemicolon
  \KwInput{$\bX$ matrix of keypoint matches, $T$: number of iterations, $\{\tau_t\}_{t=1}^T$: threshold in each iteration, $\tau$: threshold for computing the output, $q, r, s$: statistic powers}
  \KwOutput{$\bY$ filtered keypoint matches}
  $\bY \gets \bX$\;
   \For{$i=1,\dots,T$}{
   $\bS_1 = \bY^q$, $\bS_2=\bY^r\bD\bY^s$ \;
    $\bS=\bS_1\oslash(\bS_1+\bS_2)\odot \bX$ \;
    $\bS = \mathbf{1}(\bS > \tau_t)$ (optional)  \;
    $\bY \gets \bS$ \;
   }
   $\bY \gets \mathbf{1}(\bY > \tau)$
\caption{Filtering by Cluster Consistency (FCC)}
\label{alg:fcc}
\end{algorithm2e}

In practice we find that FCC with soft reweighting (no iterative thresholding) works best in general. In this case, the only parameters left are the number of iterations and the final threshold $\tau$. Allowing different $\tau$'s makes FCC a very flexible algorithm, which we explain later in \S\ref{sec:epfl}. We recommend 10 iterations for soft-reweighting. However, there are two cases where iterative-hard thresholding (the optional step) can be useful. First, for highly-corrupted full-permutation synchronization datasets, hard-thresholding can slightly improve the accuracy, as we explain in \S\ref{sec:willow}. Second, for large datasets, where we want to have minimal passes through the whole data, hard-thresholding the bad edges may accelerate the convergence and fewer iterations are needed (see \S\ref{subsec:exptourism}). 
For the midsize datasets we found that at least four iterations with $\tau_t=0.05t$ are sufficient, and for large-size datasets, we have found that only two iterations with $\tau_t=0.1 t$ are sufficient.

\subsection{On the Complexity of FCC}
\label{sec:complexity}

To calculate our statistics, we note that all matrices in the products are sparse. However, intermediate matrices, such as $\bX^2 \bD$, may be dense. To solve this issue, we notice that we only need to compute $\bS$ at elements $(i,j)$ such that $\bX(i,j)>0$. For this purpose, we use the following formula, which is proved in the supplementary material.
\begin{lemma} \label{lemma:Scomp}
For any $ij\in E$, and $q = r + s$, 
\begin{align}\label{eq:stat}
    \bS_1(i,j) = \sum_{l\in [n]}\sum_{\substack{k_1 = k_2\\ k_1, k_2 \in I_l}}\bX^r(i,k_1)\bX^s(k_2,j),\nonumber\\
     \bS_2(i,j) =\sum_{l\in [n]}\sum_{\substack{k_1 \neq  k_2\\ k_1, k_2 \in I_l}}\bX^r(i,k_1)\bX^s(k_2,j).
\end{align}
\end{lemma}
Note that $\bS_1$ and $\bS_2$ respectively correspond to the complementary cases $k_1=k_2$ and $k_1\neq k_2$, where $k_1$ and $k_2$ are indices of keypoints within the same image. This strong relationship leads to efficient computation of $\bS$.
First, we notice that
\begin{equation}\label{eq:compS1}
     \bS_1(i, j) = \left \langle \bX^r(:,i), \bX^s(:, j) \right \rangle, 
\end{equation}
i.e., $\bS_1(i, j)$ is the dot product of the $i$th and $j$th columns of the symmetric matrices $\bX^r$ and $\bX^s$. 
On the other hand,
\begin{equation}\label{eq:compS1S2}
    \bS_1(i, j) + \bS_2(i, j)  = \sum_{l\in [n]}  \left[ \Big( \sum_{k \in I_l} \bX^r(i,k)\Big) \Big( \sum_{k \in I_l} \bX^s(k,j) \Big)  \right].
    \end{equation}
By stacking the sparse elements corresponding to nonzero $\bX$ values into a vector,~\eqref{eq:compS1} and~\eqref{eq:compS1S2} can be efficiently parallelized at the cost of having computed $\bX^r$ and $\bX^s$ and sufficiently large memory. To compute~\eqref{eq:compS1S2}, we have an additional for-loop over $[n]$, but this is still efficient because $n$ is relatively small in our examples. 

Our method becomes more efficient when $r$ and $s$ are sufficiently small so that the powers $\bX^r$ and $\bX^s$ are sparse. In particular, if $\nnz(\bX)$ denotes the number of nonzero entries in $\bX$ and $n_q$ denotes the average non-zero entries of $\bX^q$ per column, then the time to compute the desired entries of $\bS_1 \odot \bX$ and $(\bS_1 + \bS_2)\odot \bX$ are both $O(\nnz(\bX) (n_r + n_s))$. 
To precompute $\bX^r$ and $\bX^s$, we need $O(\nnz(\bX) \cdot \max(\{n_l : l \leq \max(r,s)-1\}))$.
The overall space complexity is $O(N\max(\{n_l : 1 \leq l \leq \max(r,s)-1\})$.
Notice that we have the elementary bound $\nnz(\bX) \leq n N$, and consequently, the time and space complexity of Algorithm \ref{alg:fcc} are $\leq N n \max(n_r, n_s)$ and $N (n_r + n_s)$, respectively, and these are worst-case bounds. In the supplementary material, we prove the bound $n_2 < n^2$ (for $r=s=2$) for an Erd\"os-R\'enyi model with $p=n/N$.

While powers of a matrix may not be sparse in general, we observe that $\bX^r$ is sparse for $r$ sufficiently small. In particular, we generally use $r=s=2$ in our experiments, and in these cases $\bX^2$ is observed to be very sparse.

\section{Numerical Experiments}
\label{sec:numeric}

We conduct experiments on synthetic and real datasets to verify the effectiveness of the proposed FCC algorithm. In \S\ref{subsec:conv}, we demonstrate on synthetic data the improvement of the classification accuracy by our iterative procedure. Then, in Sections~\ref{sec:epfl},~\ref{sec:willow} and~\ref{subsec:exptourism}, we test the performance of our method on the EPFL\cite{EPFL_data}/Middlebury\cite{Middlebury_data}, Willow~\cite{cho2013} and Photo Tourism~\cite{photo_tourism} datasets. We report the specifications of machines on which we ran the experiments in \S\ref{sec:specs} of the supplementary material.

\begin{table*}[t]
\centering 
\resizebox{2\columnwidth}{!}{
\renewcommand{\arraystretch}{1.3}
\tabcolsep=0.1cm
\begin{tabular}{|l||c|c||c|c||c|c|c|c||c|c|c|c||c|c|c|c||c|c|c||c|c|c||c|c|c||c||}
\hline
 \multirow{ 2}{*}{Dataset} & \multicolumn{2}{c||}{ }& \multicolumn{2}{c||}{\multirow{ 2}{*}{Input}} &
\multicolumn{4}{c||}{\multirow{ 2}{*}{MatchEig}} & \multicolumn{4}{c||}{\multirow{ 2}{*}{Spectral}} &
\multicolumn{4}{c||}{\multirow{ 2}{*}{MatchALS}} &
\multicolumn{10}{c||}{FCC (ours)} 
  \\
&  \multicolumn{2}{c||}{ } &  \multicolumn{2}{c||}{ } &  \multicolumn{4}{c||}{ } & \multicolumn{4}{c||}{ } & \multicolumn{4}{c||}{ } &  \multicolumn{3}{c||}{ $\tau = 0.5$ } &  \multicolumn{3}{c||}{ $\tau = 0.9$ } &  \multicolumn{3}{c||}{ $\tau = 0.99$ } &
  \\\hline
& $n$ & $\hat m                 $ & JD & PR & JD  & PR  &  \#M & RT  &  JD  & PR  &  \#M & RT& JD  & PR  &  \#M & RT& JD  & PR  &  \#M &  JD  & PR  &  \#M & JD  & PR  &  \#M & RT
\\\hline
Dino Ring & 48 & 340 & 25.9 & 74.1 & 44.7 & 93.6 & 46 & 21 &  32.5 & 84.4 & 68& 38 & 26.8 & 85.0 & 73 & 12010 & \textbf{23.2} & 78.6 & 92 & 35.8 & 90.1 & 57& 56.2 & \textbf{94.0} & 35 & \textbf{3}
\\\hline
Temple Ring &47 &396 & 27.4 & 72.6 & 51.8 & 90.1 & 41 &37 & 36.1 & 81.9 & 66 & 62 & 30.3 & 82.2 & 73 & 16137& \textbf{25.9} & 75.5 & 94 & 35.2 & 88.2 & 58 & 49.5 & \textbf{92.4} & 41 & \textbf{2}

\\\hline
\hline

Herz-Jesu-P25 & 25 & 517 & 10.4 & 89.6 & 27.0 & \textbf{94.5} & 72 & 60 &21.8 & 92.3 & 81 & 105 & 18.5 & 93.3 & 83 & 9199 & \textbf{9.7} & 90.7 & 98 & 18.1 & 93.6 & 83 & 35.1 & 94.3 & 64 & \textbf{1}

\\\hline

Herz-Jesu-P8 & 8 & 386 & 5.7 & 94.3 & 7.1 & 95.3 & 96 & 2 & 18.6& 95.0 & 84 & 5 & 25.1 & \textbf{95.9} & 76 & 155 & \textbf{5.4} & 94.6 & 99 & 12.8 & 95.6 & 90 & 17.4 & 95.8 & 84 & $<$\textbf{1}

\\\hline

Castle-P30 & 30 & 445 & 28.2 & 71.8 & 41.2 & 85.1 & 55 & 60 & 32.7& 80.5 & 72 & 98 & 29.3 &  80.4 & 76 & 13583 & \textbf{25.5} & 76.3 & 91 & 35.8 & 87.3 & 58 & 52.0 & \textbf{89.8} & 41 & \textbf{2}

\\\hline

Castle-P19 & 19 & 314 & 29.9 & 70.1 & 43.2 & 80.4 & 58 & 18 & 32.1& 77.8 & 76 & 18 & 34.2 & 77.0 & 74 & 1263 & \textbf{27.0} & 74.2 & 92 & 34.4 & 86.5 & 59 & 45.5 & \textbf{88.2} & 47 & $<$\textbf{1}

\\\hline

Entry-P10 & 10 & 432 & 24.6 & 75.4 & 25.4 & 81.1 & 84 & 19 & 27.7 & 81.4 & 80 & 22 & 35.2 & 77.3 & 77 & 322 & \textbf{24.1} & 77.9 & 94 & 37.0 & 88.0 & 59 & 45.8 & \textbf{88.6} & 50 & $<$\textbf{1}

\\\hline

Fountain-P11 & 11 & 374 & 5.8 & 94.2 & 12.3 & 95.8 & 90 & 14 & 11.1 & 95.5 & 92 & 11 & 20.2 & 95.7 & 82 & 333 & \textbf{5.6} & 95.0 & 99 & 15.0 & 95.9 & 87 & 21.6 & \textbf{96.4} & 79 & $<$\textbf{1}

\\\hline

\end{tabular}}
\caption{Performance on the Middlebury and EPFL datasets. $n$ is the number of cameras; $\hat m$, the approximated $m$, is twice the averaged $m_i$ over $i\in [n]$; JD  and PR respectively refer to the Jaccard distance (the lower the better) and the precision rate (the higher the better) in \eqref{eq:metric_synthetic} in percentage; \#M is the ratio in percentage between the number of refined matches and the number of initial matches; RT is runtime in seconds. } 
\label{tab:real2}
\end{table*}

\subsection{Convergence of the Iterative Procedure}
\label{subsec:conv}

We generate a novel synthetic dataset that models keypoint matches with $100$ 3D scene points uniformly distributed on the unit sphere, and $100$ Gaussian-distributed cameras. Synthetic keypoints are generated by projecting the 3D points onto the image plane. Two keypoints are connected if and only if they correspond to the same 3D point. To generate the corrupted keypoint matches, with probability $0.5$ we independently replace an existing match with a false match.  For simplicity, run the default FCC with $q=4$ and $r=s=2$ without thresholding. 
Figure \ref{fig:hist} demonstrates the histogram of the  FCC statistics after 1 and 5 iterations. 
We can see that even though a large fraction of matches are missing and corrupted, the FCC statistic at the first iteration (the values in $\bS$) already achieves good separation of good and bad matches and there is only a small overlapping area between the two histograms. After only 5 iterations, the FCC statistic obtains a clean separation of good and bad matches that nicely concentrate around $1$ and $0$, respectively. Therefore, thresholding at $0.5$ (or in a large interval around it) gives exact classification of good and bad matches.
\begin{figure}[h]
    \centering
    \includegraphics[width=0.9\columnwidth]{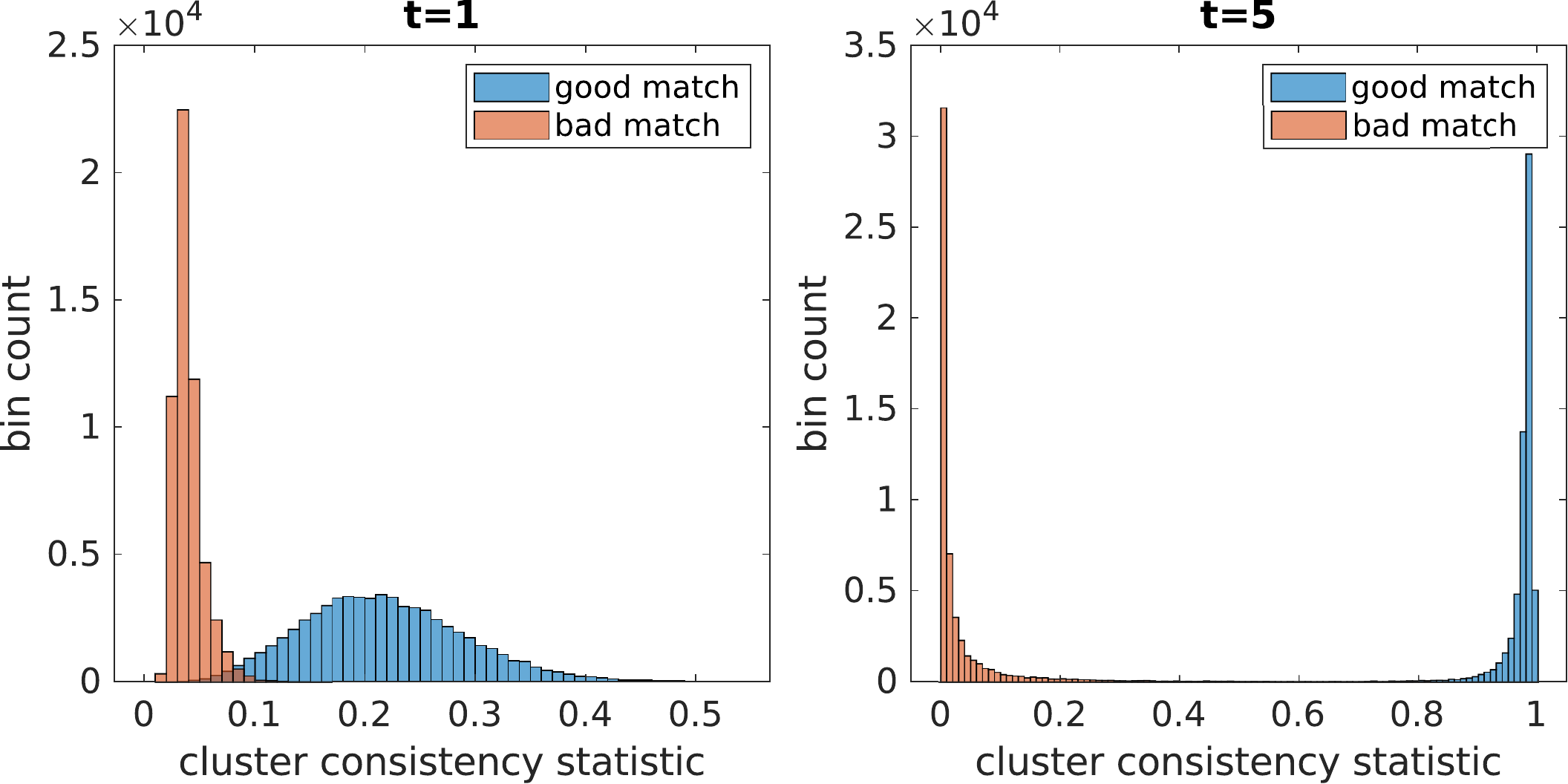}
    \caption{The histograms of the FCC statistics without thresholding for good and  bad matches after 1 (left) and 5 (right) iterations.}
    \label{fig:hist}
\end{figure}
We refer the readers to \S\ref{subsec:expsynth} in the supplementary material for details of the synthetic model, and the comparison of speed and accuracy among different algorithms under different model parameters.

\subsection{Experiments on EPFL and Middlebury}\label{sec:epfl}
We follow the experimental setup of \cite{MatchEig} and compare \cite{deepti, MatchEig, MatchALS} with our method. Each dataset consists of 8 to 48 images.
The Dino Ring and Temple Ring belong to the Middlebury database, and are subsets of the Dino (363 images) and Temple (312 images) datasets. The latter datasets are still much smaller than Photo Tourism datasets in \S\ref{subsec:exptourism} due to their small universe size (around hundreds) compared to that of Photo Tourism ($>10^4$), even though they have similar numbers of images. We do not use the whole datasets of Temple and Dino, since MatchALS cannot handle them and Spectral and MatchEig have similar results on Temple Ring and Dino Ring (see \cite{MatchEig}), and running on those subsets is much faster.

The initial matches between pairs of images are generated by nearest neighbor and ratio test using SIFT features followed by RANSAC refinement. The implementation of Spectral \cite{deepti}, MatchALS \cite{MatchALS} and MatchEig \cite{MatchEig} is exactly the same as in \cite{MatchEig}. Since the ground truth size of the universe $m$ is unknown, we follow \cite{MatchEig} and approximate $m$ by twice the average of $m_i$ over $i\in [n]$. We denote this approximate size of the universe by $\hat m$. Since Spectral, MatchEig and MatchALS  require the rank estimate of $\bX$, we follow \cite{MatchEig}  and use $\hat m$ for Spectral and MatchEig, and $2\hat m$ for MatchALS. We note that FCC does not require this parameter.  We run FCC with soft reweighting (no iterative thresholding) for 10 iterations. We note that one still needs to threshold the statistics matrix $\bS$ in the final step to obtain the refined matching. We test FCC with $\tau=0.5, 0.9, 0.99$. The higher the threshold, the more sparse the resulting match is. 

Following \cite{MatchEig}, when evaluating the estimated matches, a match is good if the epipolar constraint approximately holds given the  two keypoints and ground truth camera parameters. We report two types of metrics, the precision rate (PR) and Jaccard distance (JD) of classification:
\begin{align}
    \text{PR} = |\hat E \cap E_g|/|\hat E|,\quad
     \text{JD} = 1-|\hat E \cap E_g|/|\hat E \cup E_g|, \label{eq:metric_synthetic}
\end{align}
where $\hat E$ is the estimate of $E_g$.
We note that Jaccard distance is a more balanced metric that considers both precision and recall (note that the Jaccard distance is a decreasing function of the F-score). 

We remark that different metrics may fit better different tasks. The PR metric may be more useful for SfM tasks with initial dense matches, since as long as the refined match is good then one can reliably compute fundamental matrices (so one does not care how many good matches were thrown out). However, the combination of both precision and recall may fit better with other tasks, such as permutation synchronization, where all images share the same set of keypoints. In this case, a more natural metric is the JD.
In addition to PR and JD, we report the percentage of initial matches that remains after refinement by different algorithms (\#M in Table \ref{tab:real2}), which partially reflects recall. We also report runtimes of different algorithms in seconds (RT in Table \ref{tab:real2}).

\begin{table*}[t]
\centering 
\resizebox{1.4\columnwidth}{!}{
\renewcommand{\arraystretch}{1.05}
\tabcolsep=0.1cm
\begin{tabular}{|c||c|c|c|c|c|c|c|c|c|c|}
\hline
 & \multicolumn{1}{c|}{$n$}& \multicolumn{1}{c|}{$N$}&
\multicolumn{1}{c|}{Input}&   \multicolumn{1}{c|}{Spectral} &
 \multicolumn{1}{c|}{PPM}& \multicolumn{1}{c|}{MLift} &
 \multicolumn{1}{c|}{MALS}& \multicolumn{1}{c|}{IRGCL} & \multicolumn{1}{c|}{FCC} & \multicolumn{1}{c|}{FCC+PPM}\\
\text{Datasets}& & & &\cite{deepti} & ~\cite{PPM_vahan}&~\cite{chen_partial} & ~\cite{MatchALS}&~\cite{irgcl} &  & \\\hline
Car&40& 400 &0.52 &0.36 & 0.26 & 0.26 &0.28 & 0.25 &0.31 &\textbf{0.24} \\\hline
Duck & 50 &500 & 0.57 & 0.34 & 0.34 & 0.33& 0.34 &\textbf{0.30} &0.40 &0.35 \\\hline
Face &108& 1080& 0.14 & \textbf{ 0.041} & 0.046 &0.054 &0.055 &0.048&0.057 & 0.046\\\hline
Motorbike &40 &400 & 0.7 & 0.65 & 0.61 &\textbf{0.57} &\textbf{0.57} &0.63  &0.64 & \textbf{0.57} \\\hline
Winebottle & 66 & 660 & 0.48 & 0.29& 0.26& 0.25 &0.25 &0.24 & 0.28 & \textbf{0.23}
\\\hline
\end{tabular}
}
\caption{Matching performance comparison using the Willow database. Note that $m=10$ and $N=10n$. }\label{tab:willow}
\end{table*}

Table \ref{tab:real2} indicates superior performance of FCC in comparison to other methods in terms of both accuracy and speed.
We first note that FCC is in general 10x - 100x faster than the current fastest approaches MatchEig and Spectral, and is about 5000x - 10000x faster than MatchALS, on the datasets of EPFL and Middlebury. Moreover, FCC with $\tau=0.5$ achieves the best Jaccard distance, namely the classification error. Choosing higher $\tau$  removes more matches, which corresponds to higher precision and lower recall. When choosing $\tau=0.99$, FCC yields better precision than other methods and still maintains around 50\% of matches on most datasets. Thus, depending on the tasks, one can choose either $\tau=0.5$ or $\tau=0.99$ (or some other values) to encourage better performance on either Jaccard distance or precision (or balance between the two). For this reason, we find that FCC is a very flexible algorithm that can be suitable for different tasks, and for both choices FCC achieves the best performance than other methods.

We also find that MatchEig generally performs better than Spectral. The reason is that MatchEig is less sensitive to the parameter $\hat m$, as explained in \cite{MatchEig}. MatchALS is the slowest algorithm, and its error and precision are worse than MatchEig and similar to Spectral. We note that MatchALS only achieves better PR than FCC on Herz-Jesu-P8, where it maintains fewer matches (76\%) than FCC (84\%).

\subsection{Experiments on the Willow Database}
\label{sec:willow}

We test FCC  on the Willow database~\cite{cho2013} for multi-object matching. It includes 5 small datasets, where each dataset consists of dozens of images taken from similar viewing directions and each image contains 10 keypoints of the same 10 3D scene points. The ground truth pairwise matches are all $10 \times 10$ full-permutation matrices (rows and columns sum to 1), so common permutation synchronization algorithms can be applied. We use the initial matches provided by \cite{ConsistentFeature}
(these matches were obtained by rounding the similarity matrix of CNN features of the keypoints, however, we do not compare with \cite{ConsistentFeature} since it uses additional geometric information from keypoint coordinates).  We run FCC with $T=10$, $\tau_t=0.05t$ and $\tau =0.5$. We compared with the following methods for permutation synchronization: Spectral  \cite{deepti}, PPM \cite{PPM_vahan}, IRGCL \cite{irgcl}, MatchLift \cite{chen_partial} and MatchALS \cite{MatchALS}. Since they have the advantage of using the permutation synchronization model and they use the parameter $m$, we also tested FCC as an initializer to one of these methods, PPM. We preferred PPM since it seems more sensitive to initialization. We refer to the combined method, which first removes a small fraction of matches with low values of the FCC statistic and then applies PPM, as FCC+PPM. For the combined algorithm, we run FCC with $T=4$ only, $\tau_t=0.05t$ and $\tau=0.1$. By choosing such a low threshold, we only remove the extremely suspicious matches, while keeping the majority of matches. The reason for doing this is that the dataset is extremely noisy, and a large threshold would remove some good edges so that pairwise matches may be too sparse and will not give rise to permutations.  Thus a large threshold for FCC will degrade any follow-up algorithm for full-permutation synchronization, for which we use PPM. We use the metric of \eqref{eq:metric_synthetic} to measure accuracy. Table \ref{tab:willow} summarizes the results. We note that Spectral, PPM and IRGCL directly use the special structure of full-permutations and rely on the Hungarian algorithm to project the estimates to full-permutations. In contrast, FCC does not make these assumptions since it is designed for more general PPS. Without these additional information ($m$ and full-permutations), mere FCC is a reasonable algorithm that is roughly comparable to  Spectral.  We also note that FCC significantly improves PPM and the combined algorithm outperforms on average the rest of the algorithms.

We remark that the thresholding procedure in FCC is helpful in Willow. The reason is that graphs for the Willow datasets are relatively dense and the FCC statistics for most good matches are strictly above 0. Thus, we take a conservative strategy by using a small threshold (0.05) in the first iteration to make sure that only bad matches are removed and then gradually increase the threshold in each iteration.

\begin{table*}[h]
\centering 
\resizebox{1.7\columnwidth}{!}{
\renewcommand{\arraystretch}{1.3}
\tabcolsep=0.1cm
\begin{tabular}{|l||c|c||c|c|c|c|c|c||c|c|c|c|c|c|c|}
\hline
Algorithms & \multicolumn{2}{c||}{}& \multicolumn{6}{c||}{LUD} &
\multicolumn{7}{c|}{FCC+LUD} 
  \\\hline
\text{Dataset}& $n$ & $N$  & $n$& {\large $\hat{e}_R$} & {\large$\tilde{e}_R$} & {\large $\hat{e}_T$} & {\large$\tilde{e}_T$}  & $T_{\text{total}}$& $n$ &
{\large $\hat{e}_R$} & {\large$\tilde{e}_R$} & {\large $\hat{e}_T$} & {\large$\tilde{e}_T$} &  $T_{\text{FCC}}$ &  $T_{\text{total}}$
\\\hline
Alamo & 570 & 606963 
& 557 & 20.90 & 16.10 & 8.17 & 5.18 & 7945.9
& 538 & \textbf{19.16} & \textbf{15.23} & \textbf{7.81} & \textbf{4.92} & 755.1 & 8788.5
\\\hline

Ellis Island& 230 & 178324
& 223 & 2.16 & 1.16 & 22.99 & \textbf{21.95} & 1839.2
& 218 & \textbf{1.87} & \textbf{1.02} & \textbf{22.78} & 22.20 & 79.6 & 1904.6
\\\hline

Gendarmenmarkt& 671 & 338800
& 652 & \textbf{40.14} & 9.30 & \textbf{38.55} &  \textbf{18.33} & 3527.7
& 625 & 40.31 & \textbf{8.48} & 38.82 & 18.58 & 122.9 & 3746.1
\\\hline

Madrid Metropolis & 330 & 187790
& 315 & 13.35 & 9.37 & 12.80 & 6.87 & 1579.8
& 297 & \textbf{11.51} & \textbf{6.87} & \textbf{11.69} & \textbf{4.99} & 44.2 & 1615.9

\\\hline
Montreal N.D.& 445 & 643938 
& 439 & 2.55 & 1.06 & 1.51 & \textbf{0.66} & 5078.9
& 416 & \textbf{1.72} & \textbf{0.89} & \textbf{1.34} & 0.67 & 550.5 & 5630.6

\\\hline
Notre Dame & 547 & 1345766 
& 545 & 3.72 & 1.44 & 1.45 & 0.41 & 11315.2
& 534 &\textbf{3.46} & \textbf{1.40} & \textbf{1.35} & \textbf{0.40} & 4103.3 & 14965.5
\\\hline
NYC Library&  313 & 259302 
& 306 & 4.05 & 2.22 & 7.09 & 2.81 & 1495.9
& 283& \textbf{3.43} & \textbf{2.07} & \textbf{6.48} & \textbf{2.43} & 47.6 & 1536.5

\\\hline
Piazza Del Popolo & 307 & 157971 
& 300 & 6.92 & 3.98 & 6.78 & 2.42 & 1989.9
& 243 &\textbf{1.76} & \textbf{0.90} & \textbf{2.36} & \textbf{1.30} & 74.4 & 2053.9

\\\hline
Piccadilly&  2226 &  1278612
& 2015 & 8.11 & 3.77 & 5.41 & 2.98 & 21903.3
& 1928 &\textbf{7.17} & \textbf{3.67} & \textbf{5.04} & \textbf{2.80} & 1190.7 & 26750.7
 
\\\hline
Roman Forum&  995 & 890945
& 971 &\textbf{6.64} & \textbf{5.02}& \textbf{12.67} & \textbf{5.60} & 4858.0
& 906 & 6.67 & 5.39 & 13.29 & 5.66 & 153.6 & 5315.5

\\\hline
Tower of London& 440 & 474171 
& 431 &\textbf{6.89} & 4.29 & 21.47 & 6.85 & 1759.2
& 408 &7.17 & \textbf{4.12} & \textbf{19.14} & \textbf{6.52} & 35.5 & 1825.1

\\\hline
Union Square& 733 &  323933
& 663 &10.77 & 6.93 & \textbf{14.52} & 10.49 & 1950.6
& 599 &\textbf{8.31} & \textbf{5.91} & 14.72 & \textbf{10.12} & 33.7 & 2083.8

\\\hline
Vienna Cathedral&  789 &  1361659
& 758 & 6.58 & 3.12 & 14.52 & 8.28 & 10866.0
& 682 & \textbf{3.76} & \textbf{1.92} & \textbf{11.60} & \textbf{6.79} & 2845.6 & 13550.2
\\\hline
Yorkminster & 412 &  525592
& 407 & 4.25 & 2.71 & 6.45 & 3.68 & 2267.3
& 386 &\textbf{4.04} & \textbf{2.66} & \textbf{6.00} & \textbf{3.44} & 90.1 & 2371.1

\\\hline

\end{tabular}}
\caption{Performance on the Photo Tourism datasets: $n$ and $N$ are the number of cameras and key points, respectively ($n$  is listed three times as the initial number for the dataset and the remaining numbers after removing cameras in both pipelines); $\hat e_R$ $\tilde e_R$ indicate mean and median errors of absolute camera rotations in degrees, respectively; $\hat e_T$ $\tilde e_T$ indicate mean and median errors of absolute camera translations in meters, respectively; $T_{\text{FCC}}$ and $T_{\text{total}}$ are the runtime of FCC and the total runtime of the given pipeline (LUD or FCC+LUD), respectively (in seconds).} 
\label{tab:real}
\end{table*}

\subsection{Experiments on the Photo Tourism Database}
\label{subsec:exptourism}

We test FCC on the SfM Photo Tourism database \cite{photo_tourism} with precomputed pairwise image matches provided by \cite{SenguptaAGGJSB17}. These matches were obtained by thresholding SIFT feature similarities. We compare the LUD pipeline \cite{cvprOzyesilS15} to the LUD pipeline with FCC pre-processing (denoted as FCC+LUD). All components not involving FCC were implemented identically in both pipelines. 

For FCC+LUD, we first prune those matches with the FCC method using the parameters $T=2$, $\tau_t=0.1t$ and $\tau=0.5$. We remark that typically on such SfM data, we find that most of the values of the $\bS$ statistic are either larger than 0.9 or smaller than 0.1. Therefore the output is not sensitive to the choice of $\tau$, which is away from 0 and 1. 

After filtering keypoint matches for all pairs of images using FCC, we computed the essential matrices using the least median of squares procedure. We did not apply RANSAC since it is sensitive to the choice of threshold and sample size, which introduces more randomness in our evaluation. In contrast, the least median of squares is parameter-free and much faster. We also computed the essential matrices of the LUD pipeline in the same way. To account for cases where there were not enough keypoint matches, if there were less than 16 keypoint matches remaining after FCC filtering (as required by the median least squares), we remove the correspondence between the two images. At last, we fed essential matrices into the camera pose solver in the standard LUD pipeline \cite{cvprOzyesilS15}.  
We remark that the camera location solver in the LUD pipeline automatically examines the parallel rigidity of the viewing graph and extracts the maximal parallel rigid subgraph. This subgraph extraction procedure often removes more cameras if some camera correspondences are removed by FCC beforehand. However, our experiments show that the final numbers of cameras in both our procedure and pure LUD are comparable (with at least 80\% of cameras remaining in each dataset).

Table \ref{tab:real} reports the number of cameras, the accuracy (mean and median errors of rotations and translations) and  runtime of the standard LUD pipeline and the new FCC+LUD procedure. 
From the table, we observe that FCC+LUD improves the estimation of camera parameters over LUD on the unfiltered keypoint matches in a significant portion of the datasets. In particular, the only two sets where the LUD pipeline has better accuracy overall are
Gendarmenmarkt (where the error is very high anyway) and Roman Forum (where the differences are not too significant). This is mainly due to the highly symmetric buildings contained in the two datasets which results in self-consistent bad keypoint matches. These malicious matches cannot be removed by merely exploiting the cycle-consistency, and 3D geometric information should be used to solve this ambiguity. 

One may also compare the total runtime of LUD with that of FCC+LUD and its subcomponent, FCC. For most datasets the sum of the total time of LUD and of  FCC is about the same as the total time of FCC+LUD, although the one exception is the Piccadilly dataset. We noticed that in this case it takes much more time to extract the maximal parallel rigid component (it takes 4000 seconds for FCC+LUD, while only 72 seconds for LUD). For pure LUD, the parallel rigid component seems close to the original graph; however, this is not the case for FCC+LUD.

Additional figures appear in the supplemental material. One set of figures compares the estimation errors of FCC+LUD and LUD on their common set of cameras. They demonstrate that FCC+LUD generally performs better than LUD on this set of cameras, so the advantage of FCC+LUD was not just obtained by removing bad cameras. The other set of figures compares the performance of LUD on the cameras of FCC+LUD and the rest of the cameras. The errors on the first set of cameras are generally smaller than on the second set. Thus, FCC is helpful in removing bad cameras.
We finally remark that all other standard PPS algorithms are not scalable to the Photo Tourism datasets, and thus we only apply FCC here.

\section{Conclusion}
\label{sec:conclusion}
In this work, we develop novel robust statistics for multi-object matching. These statistics are based on cycle consistency constraints on the graph with all image keypoints as nodes and keypoint matches as edges. In particular, we combine within-cluster and cross-cluster constraints into a combined statistic to yield distinguished values between corrupted and uncorrupted matches. The resulting FCC method is efficiently implementable in practice due to only requiring sparse matrix multiplication and parallelization. Experiments in synthetic and real data demonstrate state-of-the-art accuracy and speed for structure from motion tasks. In particular, FCC is the only robust multi-object matching method that can scale to city-scale SfM data.

Due to the performance of our method combined with the intriguing heuristics, one direction for future work is to theoretically explore the properties of our statistics. Further, future work will explore the incorporation of SIFT descriptor similarities in our statistics, instead of using the adjacency matrix of the keypoint matching graph alone, in order to yield even more accurate 3D reconstructions. Finally, the use of the cross-cluster constraints allows us to develop a complementary statistic that has not been thought of before in the literature. In particular, this complementary statistic is based on considering paths containing edges in the disjoint graphs $G$ and $G_D$. It would be interesting to explore applications of such disjoint graph constraints in other settings. 

\subsection*{Acknowledgement}
This work was supported by NSF award DMS 1821266.

{\small
\bibliographystyle{ieee_fullname}
\bibliography{egbib}
}


\appendix
\newpage

\section{Supplementary Material}
Section \ref{subsec:expsynth} includes the synthetic experiment. Section \ref{sec:more_experiments} provides additional experiments on the Photo Tourism database. 
Section \ref{sec:specs} lists the specifications of machines that run the different experiments. 
Section \ref{sec:prop:D} proves Proposition \ref{prop:D}. Section \ref{sec:prop:D2} proves Proposition \ref{prop:D2}. Section \ref{sec:lemma:Scomp}
proves Lemma \ref{lemma:Scomp}.  
Section \ref{sec:erdos} shows that under a special model $\bX^2$ is sufficiently sparse and consequently the sample complexity bound in this case can be improved.
Finally, \S\ref{sec:simple_example} demonstrates by a simple example the meaning of the new statistics.

\subsection{More Experiments on the Synthetic Dataset}
\label{subsec:expsynth}

We provide more details and additional experiments for the novel synthetic dataset of \S\ref{subsec:conv}. Recall that this dataset models keypoint matches with $m$ 3D scene points on the unit sphere  $\mathbb{S}^2 \subset \mathbb{R}^3$, and $n$ cameras. We first uniformly and independently sample $m$ points from $\mathbb{S}^2$. 
Next, we generate random parameters for the  cameras. We independently sample their locations from 3D  Gaussian isotropic distribution with mean zero and covariance matrix $10\boldsymbol I$. We increase the distance from the origin of each location by 1, so that all cameras are located outside $\mathbb{S}^2$. In order to generate camera rotations, we assume that each camera points towards the origin and we rotate counterclockwise the image plane of each camera around the optical axis with an angle uniformly and independently sampled between 0 and $2\pi$. For simplicity, we set all image sizes as 1000 $\times$ 1000 and all focal lengths as 500. Synthetic keypoints are generated by projecting the 3D points onto the image plane. Next, we independently sample pairs of cameras with probability $p=0.5$ and find keypoint matches between them. Two keypoints are connected if and only if they correspond to the same 3D point. We remove the camera pairs that share less than 5 common visible 3D points (clearly, their keypoint matches are also removed). To generate the corrupted keypoint matches, we independently remove  with probability $q_0$  existing matches and add  with probability $q_1$ a false match to pairs of unmatched keypoints.

To compare the accuracy with different algorithms,  We use the above model with $m=n=100$ while varying the probabilities $q_0=q_1=q$ with $q$ ranging from $0.1$ to $0.9$.  we denote by $\hat E$ the estimate of $E_g$ (that is, estimated set of good matches) and use the common Jaccard distance \eqref{eq:metric_synthetic}, which is a decreasing function of Jaccard similarity and F-score.
We run both FCC without thresholding for 10 iterations and FCC with thresholding using the following parameters: $T=10$, $\tau_t=0.05t$ and $\tau=0.5$. 
We compare the accuracy with the following algorithms: MatchALS \cite{MatchALS}, Spectral \cite{deepti} and PPM \cite{PPM_vahan}.
For runtime, we later report a comparison with MatchLift  \cite{chen_partial}.
We use the default choice of MatchALS and MatchLift with a maximum of  1000 iterations using the code of \url{https://github.com/zju-3dv/multiway}.  Both Spectral and PPM were originally developed for permutation synchronization, but they can be directly extended to PPS, which fits the current setting. For Spectral, we first compute $\bV$ whose columns are top $m$ eigenvectors of $\bX$. We then divide the rows of $\bV$ into $n$ blocks of sizes $m_i\times m$, $i \in [n]$. We then apply the Hungarian algorithm to round each rectangular block of $\bV$ to obtain a block matrix $\bP$ of absolute partial permutations. The estimated $\bX_g$ is $\bP\bP^T\odot \bX$. We apply the same rounding procedure after each power iteration of PPM. Both Spectral and PPM require $m$ as input, whereas FFC does not need it and MatchALS and MatchLift estimate it.  We did not run MatchLift on the above setting since the size of the dataset is computationally prohibitive for it. We also do not compare with \cite{hu2018distributable}, as it is simply a distributed version of MatchALS. We did not compare with  \cite{serlin2020distributed} since it requires the SIFT input for each keypoint.
Figure \ref{fig:err} shows the results of this experiment for varying $q_0=q_1=q$. 
Generally, FCC is comparable to MatchALS, and it performs better than Spectral, MatchEig and PPM. 

\begin{figure}[h]
    \centering
    \includegraphics[width=\columnwidth]{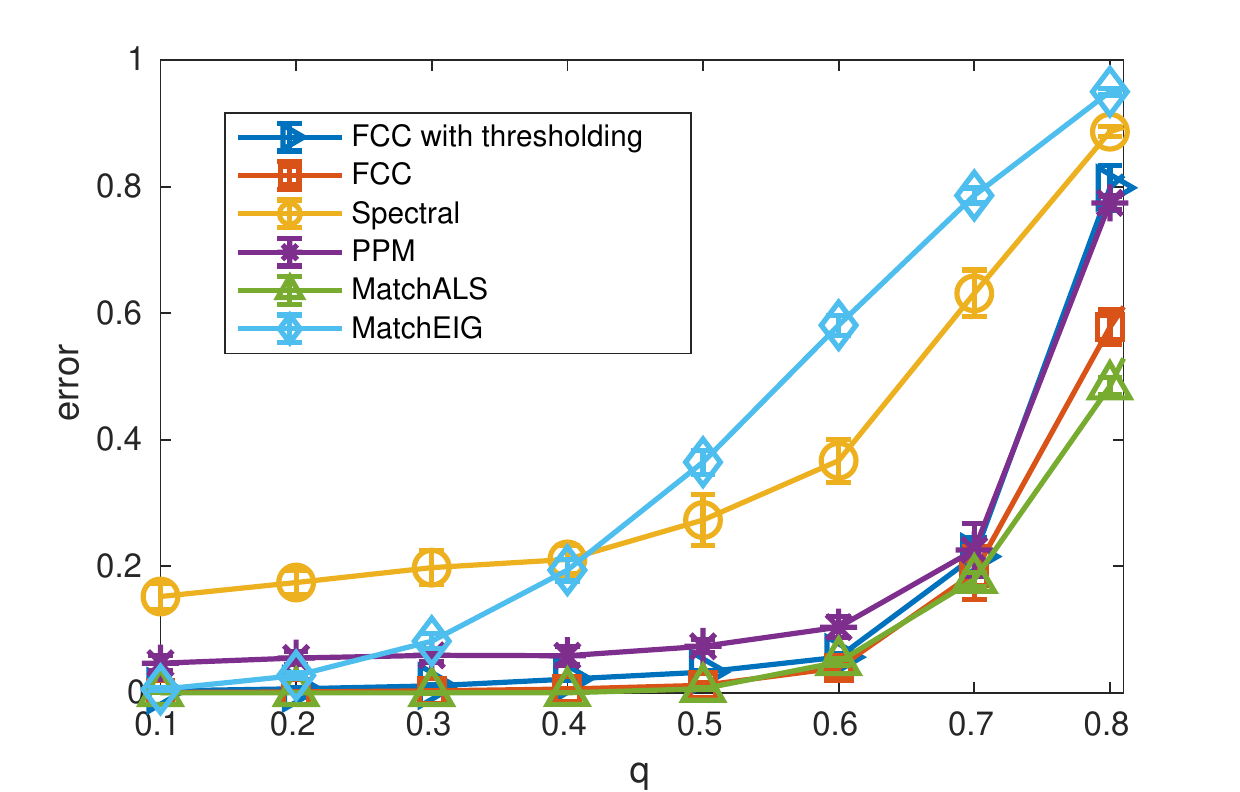}
    \caption{Matching error for different algorithms.}
    \label{fig:err}
\end{figure}

Tables \ref{tab:speed1} and \ref{tab:speed2} compare the runtime per iteration (in seconds) of the different algorithms with varying $n$ and $m$, respectively (the rest of the parameters are specified in their captions). Table \ref{tab:speed2} uses NA when an algorithm exceeds the memory limit.

\begin{table}[h]
\centering 
\resizebox{0.8\columnwidth}{!}{
\renewcommand{\arraystretch}{1}
\tabcolsep=0.1cm
\begin{tabular}{|c||c|c|c|c|c|c|c|c|c|}
\hline
 $n$& 50 & 100 &200 & 500 & 1000\\\hline
 FCC & 0.02 & 0.08 & 0.62 & 2.68 & 5.74\\\hline
 Spectral &0.41 & 2.22 & 11.08 & 62.07 & 254.57\\\hline
 MLift & 0.9 & 5.64 & 56.92 & 481.46 & 3678.1 \\\hline
 MALS & 0.15 & 0.47 & 2.01 & 11.86 & 46.89
\\\hline
\end{tabular}}
\caption{Runtime (in seconds) per iteration with varying $n$. The other parameters are $p=50/n$, $m=100$, $q_0 = 0.5$ and $q_1=0$.}\label{tab:speed1}
\end{table}

\begin{table}[h]
\centering 
\resizebox{0.8\columnwidth}{!}{
\renewcommand{\arraystretch}{1}
\tabcolsep=0.1cm
\begin{tabular}{|c||c|c|c|c|c|c|c|c|c|}
\hline
 $m$& 50 & 100 &500 & 1000 & 10000\\\hline
 FCC & 0.01 & 0.02 & 0.07 & 0.21 &1.71 \\\hline
 Spectral &5& 40.5 & 4976.4 &37234  &NA \\\hline
 MLift & 6.8 & 51.5 & 3729.5 & NA  & NA \\\hline
 MALS & 0.5 & 2.48 & 88.16 & NA & NA
\\\hline
\end{tabular}}
\caption{Runtime (in seconds) per iteration with varying $m$. We fix $n=100$, $p=0.1$, $q_0=0.5$ and  $q_1=0$.}\label{tab:speed2}
\end{table}

We note that FCC is much faster than all other methods in each iteration. In particular, it seems that the per-iteration runtime of FCC approximately scales linearly with $m$. 
We further remark that FCC typically requires only a few iterations and Spectral is not iterative (thus we reported its total time). In contrast, MatchALS and MatchLift require hundreds of iterations to converge.  
While all other algorithms are memory-demanding, FCC enjoys small space complexity and is able to run on datasets with $m \gg 1,000$. This makes FCC useful for common SfM datasets, where $m$ is often larger than $10,000$.
The SVD is the main computational bottleneck for both PPM and Spectral and we thus only report runtime for Spectral. We also skip \cite{hu2018distributable} as it is slower than Spectral and typically only achieves an order of magnitude speed-up compared to MatchALS (on a single machine).

\subsection{Additional Results on Photo Tourism}
\label{sec:more_experiments}

In order to get better information than the mere averages and medians, we consider the following two kinds of figures. First of all, we aim to compare the cameras remaining in FCC+LUD (and thus also in LUD). 
For this purpose, Figure \ref{fig:histR} and \ref{fig:histT} show the histograms of the difference of errors between LUD and FCC+LUD for this set of cameras for rotations and translations, respectively. 
They indicate positive values (where FCC+LUD improves over LUD) in blue and negative values (where LUD improves over FCC+LUD) in orange. Overall, in this set of cameras FCC+LUD often performs better.  
Second of all, we aim to understand whether the set of cameras that remained in LUD but not in FCC+LUD had higher errors and thus it might have been justified to remove them. For this purpose, 
Figures \ref{fig:histR_distribution} and  \ref{fig:histT_distribution} show the normalized histograms (having percentages instead of numbers of cameras) of LUD estimation errors for rotations and translations, respectively. In each figure two different kinds of distributions are generated. First, one generate the histogram for all cameras remaining in the LUD pipeline and normalize it (so it represents a distribution) and present it in blue; next, one generate the histogram of same errors for the additional cameras that were removed by the FCC+LUD procedure, normalize it and plots it in orange. 
Overall, the distribution of points removed by FCC+LUD seem to have higher values. For example, this property is noticed in the Tower of London and Yorkminster datasets, where on the other hand, the comparisons of Figures \ref{fig:histR} and \ref{fig:histT} do not seem advantageous for these datasets.

\begin{figure*}[htbp]
    \centering
    \includegraphics[width=2\columnwidth]{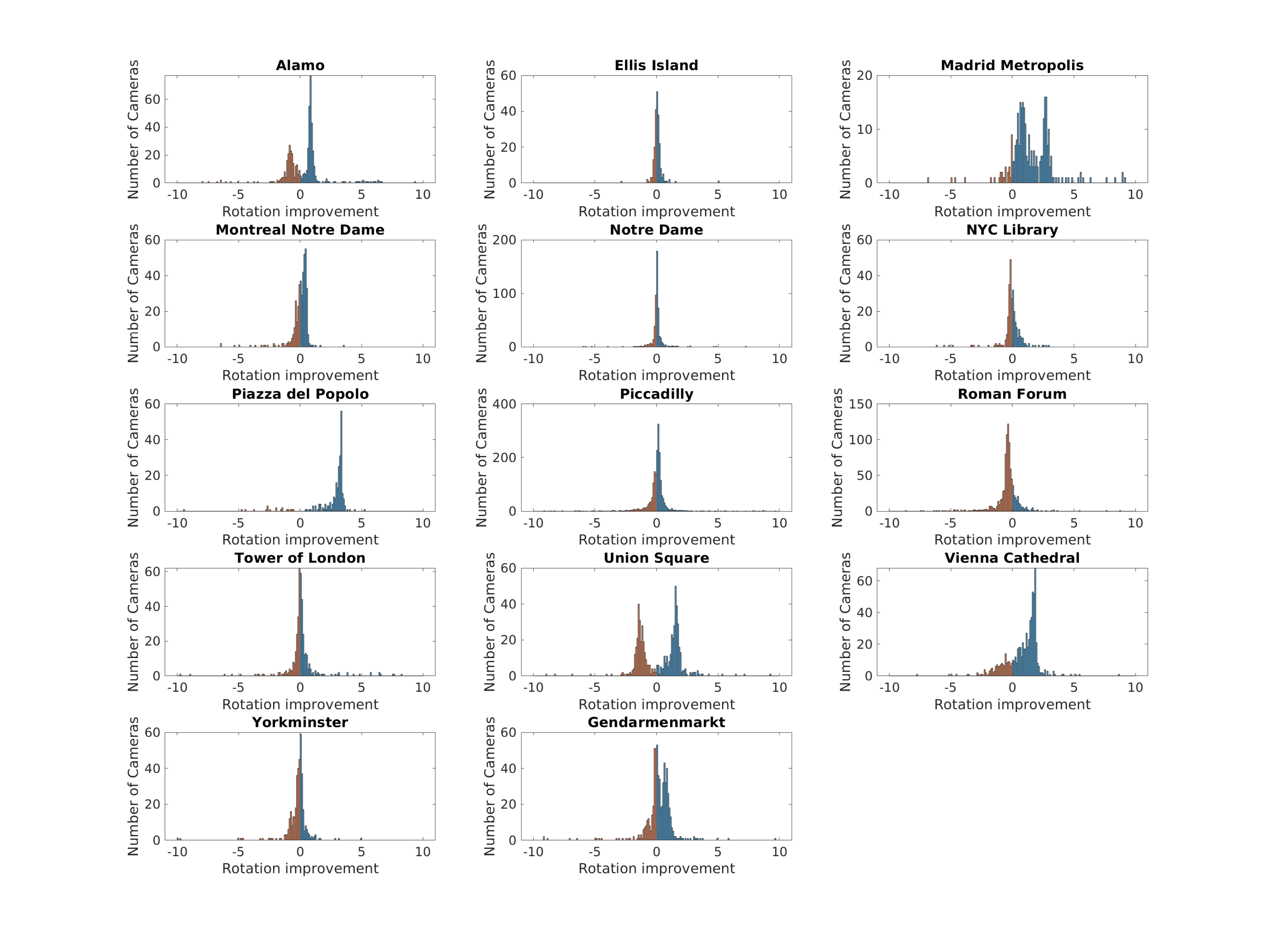}
    \caption{Histograms demonstrating general improvement of rotation errors after FCC filtering. Only the common set of cameras in the LUD and FCC+LUD pipelines were utilized. 
    The histograms are of the differences of rotation errors between LUD and FCC+LUD for this set of cameras. 
Positive values (where FCC+LUD improves over LUD) are in blue and negative values (where LUD improves over FCC+LUD) are in orange. 
}
    \label{fig:histR}
\end{figure*}

\begin{figure*}[htbp]
    \centering
    \includegraphics[width=2\columnwidth]{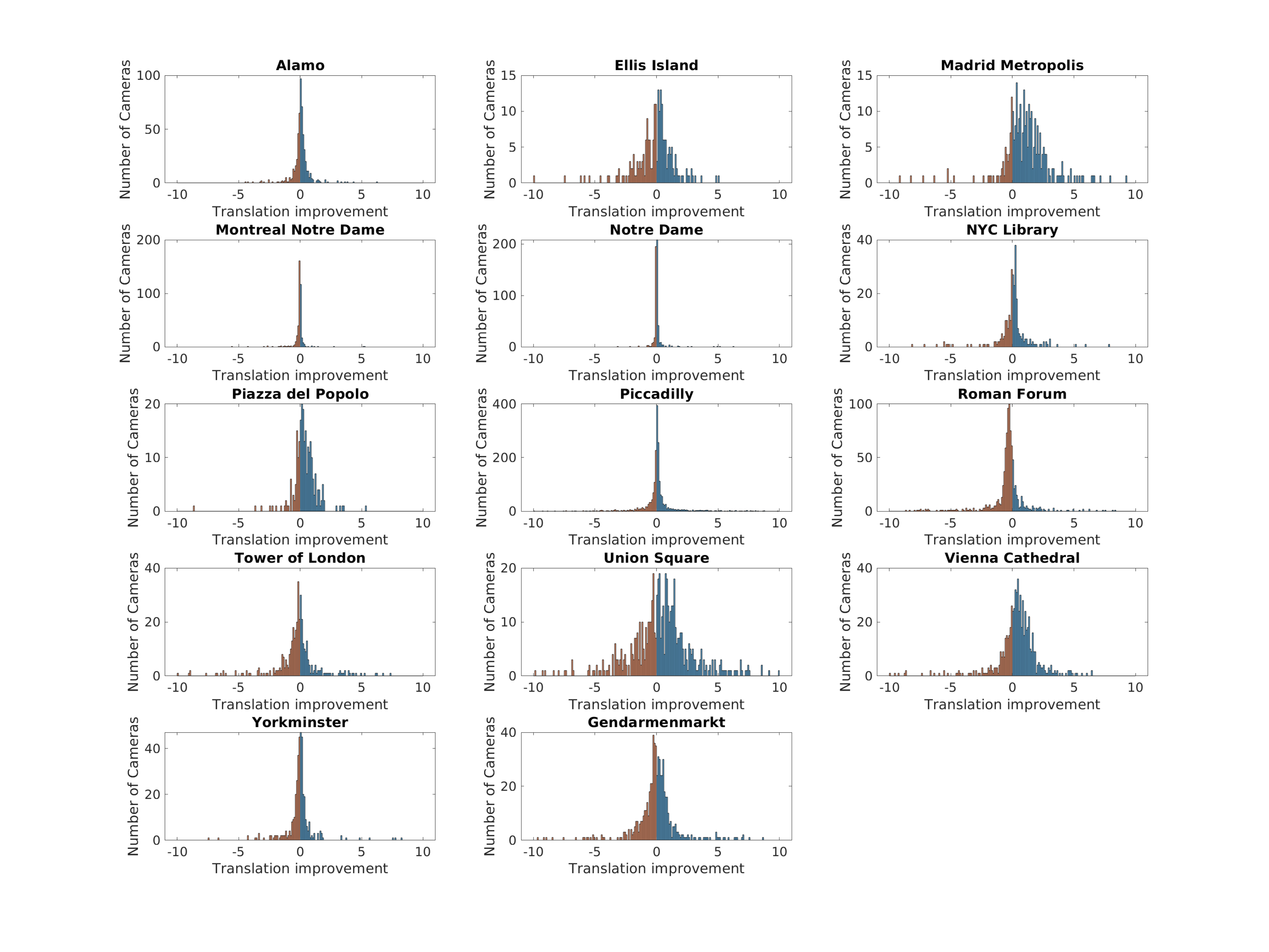}
\caption{Histograms demonstrating general improvement of translation errors after FCC filtering. Only the common set of cameras in the LUD and FCC+LUD pipelines were used. 
    The histograms are of the differences of translation errors between LUD and FCC+LUD for this set of cameras. 
Positive values (where FCC+LUD improves over LUD) are in blue and negative values (where LUD improves over FCC+LUD) are in orange. \label{fig:histT}}
\end{figure*}

\begin{figure*}[htbp]
    \centering
    \includegraphics[width=2\columnwidth]{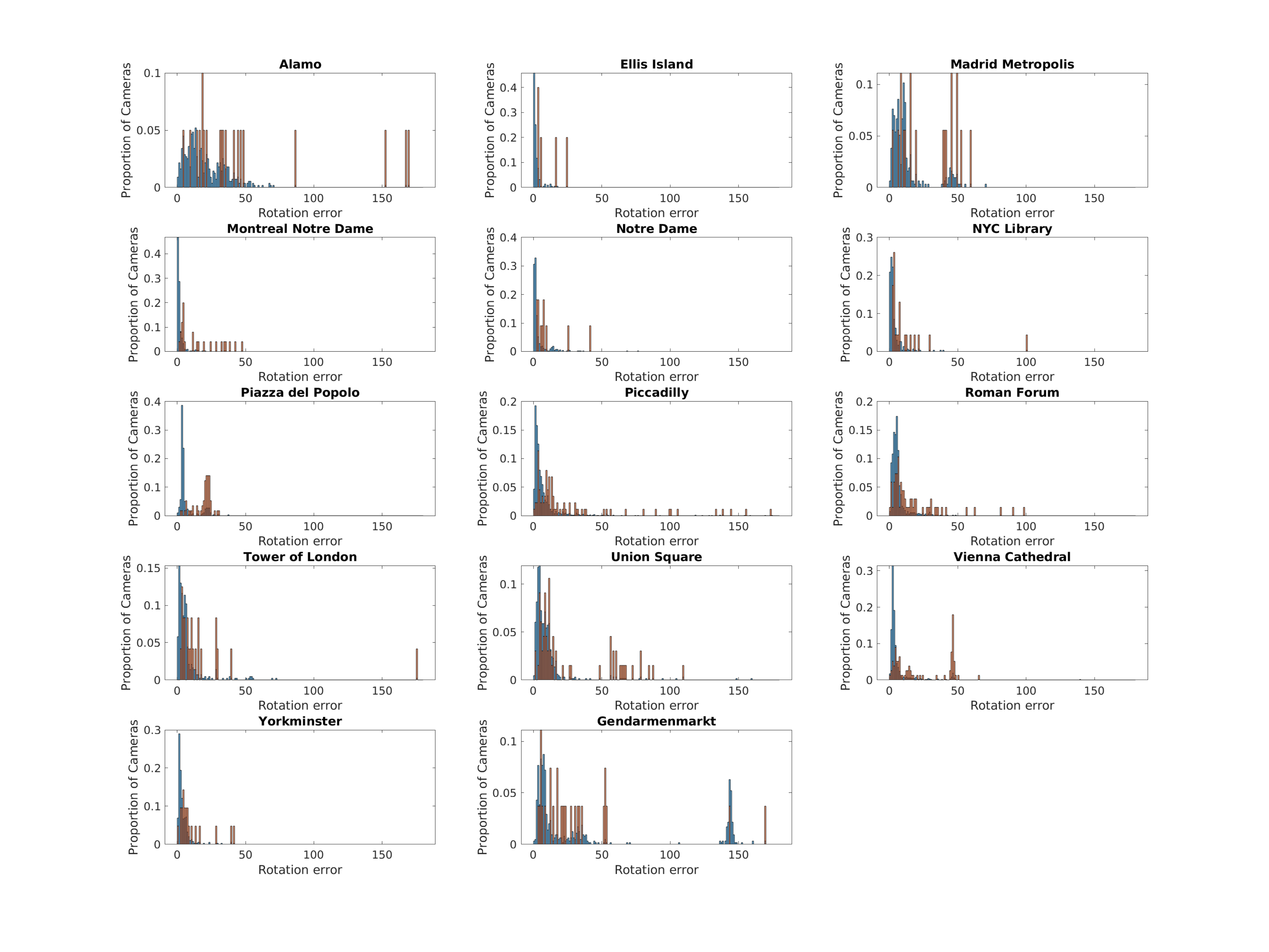}
    \caption{Demonstration of two different normalized histograms of LUD rotation errors. The blue normalized histogram consists of rotation errors of all cameras that remained in the LUD pipeline.  The orange normalized histogram consists of rotation errors of all of these cameras that were thrown away by FCC. The histograms are normalized so that the total area of each one is 1.}
    \label{fig:histR_distribution}
\end{figure*}

\begin{figure*}[htbp]
    \centering
    \includegraphics[width=2\columnwidth]{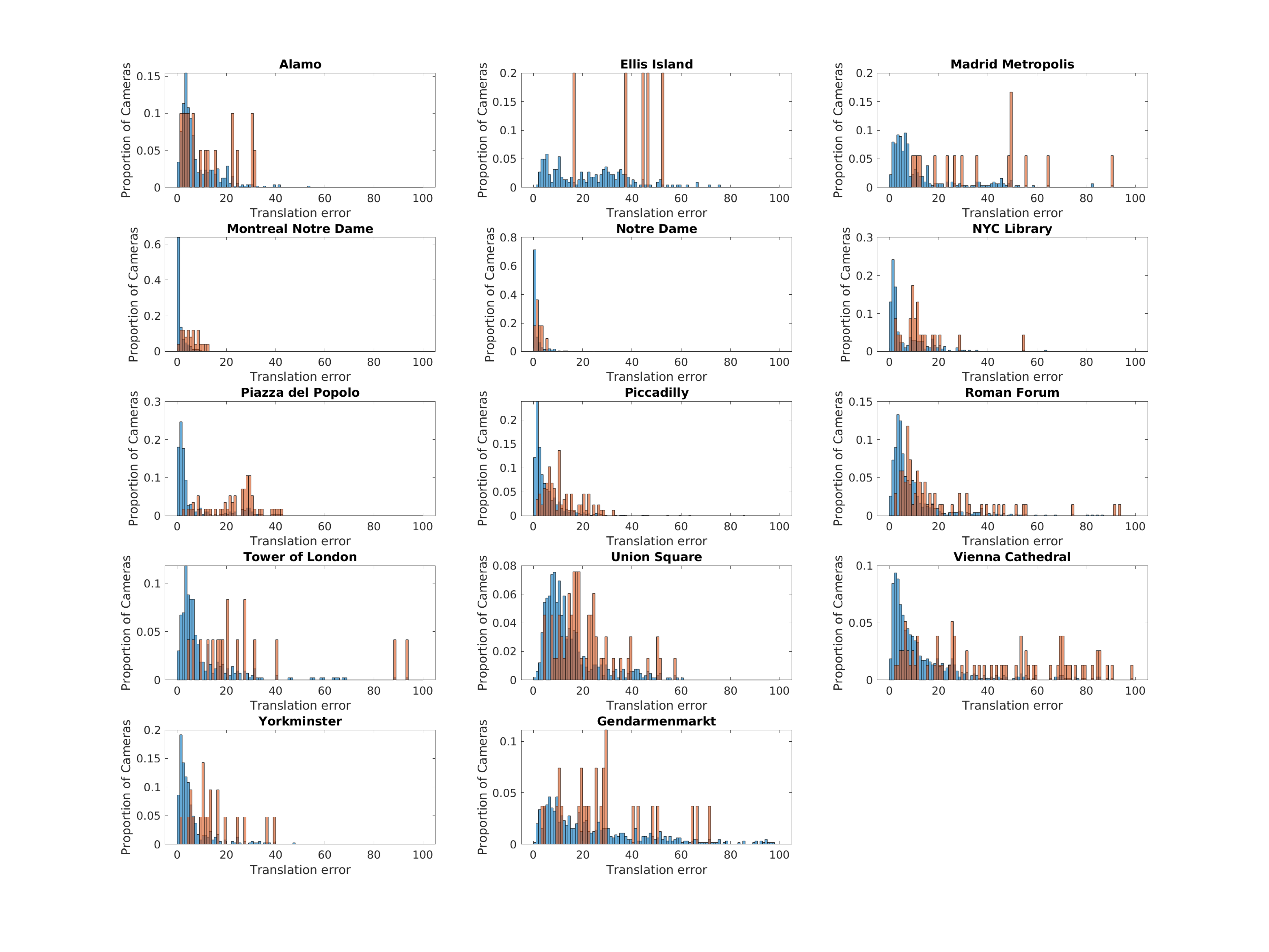}
    \caption{Demonstration of two different normalized histograms of LUD translation errors. The blue normalized histogram consists of translation errors of all cameras that remained in the LUD pipeline.  The orange normalized histogram consists of translation errors of all of these cameras that were thrown away by FCC. The histograms are normalized so that the total area of each one is 1.}
    \label{fig:histT_distribution}
\end{figure*}

\subsection{Specifications of Machines}\label{sec:specs}
\begin{itemize}
    \item The synthetic, Willow and Middlebury datasets are run on a computer with an 8-core 3.8 GHz Intel i7-10700K CPU and 48 GB RAM.  

    \item The EPFL datasets are run on a personal computer with quad-core Intel Core i5 CPU and 8 GB memory.

    \item The Photo Tourism datasets are implemented on a computer with 64 GB memory and an Intel Core i7-6850K CPU  with six 3.6GHz cores.
\end{itemize}

\subsection{Proof of Proposition \ref{prop:D}} \label{sec:prop:D}

For simplicity we only prove the case where  $r=s=1$, so that $\bS_2^{\text{sub}}=\bY\bD\bY$. The argument for other values of $r$ and $s$ is exactly the same.
It suffices to show that $\bS_2^{\text{sub}}(i,j)=0$ for any $ij\in \hat{E}^*$. Note that
\begin{align}
\nonumber
    \bS_2^{\text{sub}}(i,j) &= \sum_{k_1,k_2}\bY(i,k_1)\bD(k_1,k_2)\bY(k_2,j)\\
    &=\sum_{l\in [n]}\sum_{k_1\neq k_2 \in I_l}\bY(i,k_1)\bY(k_2,j).
    \label{eq:S1comp}
\end{align}
For $l \in [n]$, define 
$$\bS_{2,l}^{\text{sub}}(i,j):=\sum_{k_1\neq k_2 \in I_l}\bY(i,k_1)\bY(k_2,j).$$
We prove by contradiction that $\bS_{2,l}^{\text{sub}}(i,j)=0$ for each $l\in [n]$ and $ij\in \hat{E}^*$ and thus conclude that
$\bS_{2}^{\text{sub}}(i,j)=0$ for any $ij\in \hat{E}^*$. 
Assume on the contrary that there exists $l\in [n]$ such that
\begin{align}
    \bS_{2,l}^{\text{sub}}(i,j) = \sum_{k_1\neq k_2 \in I_l}\bY(i,k_1)\bY(k_2,j)>0
\end{align}
and without loss of generality assume that for some fixed $k_1\neq k_2\in I_l$, $\bY(i,k_1)\bY(k_2,j)>0$. This implies that $ik_1,jk_2\in E_{\text{sub}} \subseteq \hat{E}^*$. Also recall that $ij\in \hat{E}^*$. Therefore the edges $ik_1$, $jk_2$ and $ij$ are in the cycle-consistent graph $\hat{G}^*([N],\hat{E}^*)$. Consequently, the edge $k_1k_2$ is also in this cycle-consistent graph. However, since $k_1$ and $k_2$ are 
two keypoints in the same image, they are from two distinct connected components of $G_g$, and  thus belong to two connected components in $\hat{G}^*$. Therefore, $k_1k_2\notin \hat{E}^*$ and this results in a contradiction. 

\subsection{Proof of Proposition \ref{prop:D2}}
\label{sec:prop:D2}

By Proposition \ref{prop:D}, $\hat{\bX}^*(i,j)=1$ implies that $\bS_2^{\text{sub}}(i,j)=0$. Thus it suffices to show that $\hat{\bX}^*(i,j)=0$ implies that $\bS_2^{\text{sub}}(i,j)>0$.  Since $(i,j)$ is in the off-diagonal blocks of $\hat{\bX}^*$, node $i$ and $j$ are from two images. Since $\hat{\bX}^*(i,j)=0$, by definition of $\hat{\bX}^*$, node $i$ and $j$ are from two connected components of $\hat{G}^*$. By our assumption, there exists an image that contains the nodes $k,l$ respectively in the two components, where $\{k,l\}\neq \{i,j\}$. Therefore, there exists a path that goes through $i,k,l,j$ and thus $\bS_2^{\text{sub}}(i,j)>0$. 

\subsection{Proof of Lemma~\ref{lemma:Scomp}}
\label{sec:lemma:Scomp}

\begin{proof}
By applying $\bX^q = \bX^r \bX^s$ and respectively replacing the two $\bY$'s in ~\eqref{eq:S1comp} with $\bX^r$ and $\bX^s$, we obtain that
\begin{equation*}
\bS_2(i,j) = \sum_{l\in [n]}\sum_{k_1 \neq k_2 \in I_l}\bX^r(i,k_1)\bX^s(k_2,j).    
\end{equation*}
Similarly,
\begin{align*}
\nonumber
 \  \ \ \ \ \ \ \ \ \ \ \ \bS_1(i,j) &= \sum_{l\in [n]}\sum_{k \in I_l}\bX^r(i,k)\bX^s(k,j) \\
    &= \sum_{l\in [n]}\sum_{k_1 = k_2 \in I_l}\bX^r(i,k_1)\bX^s(k_2,j).  \  \ \ \ \ \ \
\qed
\end{align*}
\renewcommand{\qedsymbol}{}
\vspace{-\baselineskip} 
\end{proof}

\subsection{Sparsity of $\bX^2$ for an Erd\"os-R\'enyi Model}
\label{sec:erdos}

Our goal is to understand the sparsity of $\bX^2$ (and $\bX$) and consequently bound $n_2$ under a special probabilistic model, namely, the Erd\"os-R\'enyi random graph model. We first observe the following basic lemma.

\begin{lemma}
Assume an Erd\"os-R\'enyi random graph model on $N$ nodes with parameter $p$. Then, 
\begin{equation}
\label{eq:expectation_erdos}
    \mathbb E \bX^2(i,j) = (N-2) p^2.
\end{equation}
\end{lemma}
\begin{proof}
    A simple counting argument reveals that the expected number of paths of length 2 between nodes $i$ and $j$ is $p^2 \cdot(N-2)$.
\end{proof}

This lemma implies that the expected number of nonzero entries in $\bX^2$ is $ N^2(N-2) p^2$. Recall that $n_1$ and $n_2$ denote the average number of nonzero entries per column of $\bX$ and $\bX^2$. Thus, in a corresponding random graph, one could set $p = n_1 / N$, and obtain that  $\E \nnz(\bX^2) \approx n_1^2 (N-2)$ and thus $n_2 \approx (N-2)n_1^2/N \approx n_1^2$ (for simplicity we avoid the estimate of concentration around the expected value in \eqref{eq:expectation_erdos}). Therefore, at least heuristically, we could expect sparse powers of $\bX$ given sufficient sparsity of $\bX$ itself coupled with well-distributed nonzero entries.

Using the estimates of  \S\ref{sec:complexity} and the upper following bound: $\nnz(\bX) \leq N n_1$, the computational complexity of the algorithm is $O(N n_1 n_1^2 (N-2)/N) = O(N n_1^3)$.

\subsection{A Simple Demonstrating Example}
\label{sec:simple_example}
The following example heuristically illustrates the usefulness of the $\bS$ statistic in emphasizing the good edges, while rigorous theorems are left for future work.

Consider the example of 4 images that each contain the same two key points. Suppose that the perceived point matches between the images correspond to the following good (indicated by black ones) and bad matches (indicated by red ones):

\begin{equation*}
  \bX =  
\left[
\begin{array}{cc|cc|cc|cc}
0 & 0 & 0 & {\color{red} 1} & 1 & 0 & 1 & 0 \\
0 & 0 & 0 & 0 & 0 & 1 & 0 & 1\\ \hline
0 & 0 & 0 & 0 & 1 & 0 & 1 & 0\\
{\color{red} 1} & 0 & 0 & 0 & 0 & 1 & 0 & 1\\ \hline
1 & 0 & 1 & 0 & 0 & 0 & 1 & 0\\
0 & 1 & 0 & 1 & 0 & 0 & 0 & 1\\ \hline
1 & 0 & 1 & 0 & 1 & 0 & 0 & 0\\
0 & 1 & 0 & 1 & 0 & 1 & 0 & 0
\end{array}
\right].
\end{equation*}

\vspace{.3cm}
The top part of Figure \ref{fig:ex} demonstrates the corresponding keypoint graph. The $n=4$ dashed boxes indicate the images. There are $N=8$ nodes clustered to $m=2$ groups of two different scene points, colored by blue and orange. The edges in $E_g$ are in black and the single edge in $E_b$ is in red. The dashed blue edges are not in $E$, but in $E_D$. The keypoints of all pairs of images, except for the first two images, are correctly matched (that is, blue/orange nodes are matched by a black edge in $E_g$); however, keypoints in images 1 and 2 are mismatched by a single edge between an orange and blue node. The bottom of Figure \ref{fig:ex} describes an equivalent representation of this graph without specifying the underlying images.

For simplicity, assume the case of $q=2$ and $r=s=1$. Examining $\bS_1=\bX^2$ or observing length 2 paths in $G$ (i.e., of black and red edges) in the top part of Figure \ref{fig:ex} reveals that there are no such paths that connect the mismatched nodes, 1 and 4 (or alternatively, complete the bad red edge into a 3-cycle). On the other hand, the nodes of the good edges are all connected by either one or two paths of length 2 in $G$ (good edges with 2 such paths are $57$ and $68$). Examining $\bS_2=\bX \bD \bX$ or paths composed of black, dashed blue and then black edges reveals that the two nodes of the bad match have two paths each connecting them ($\{17,78,84\}$ and $\{15,56,64\}$), while the nodes of good matches (edges in $E_g$) have at most one such path (the edges with exactly one path are $17$, $15$, $46$, $48$). Using the information above, we can directly compute values of $\bS$. For the bad edge $14$, $\bS(1,4)=\bS(4,1)=0/(0+2)=0$. On the other hand, for $ij \in E_g$, a rough lower bound for $\bS(i,j)$  is $1/(1+2)=1/3$ and a more careful observation of all values indicates that there only two values of this statistic: $1/(1+1)=1/2$ (for edges $17$, $15$, $46$, $48$) and $2/(2+0)=1/(1+0)=1$. In this case, the matrix $\bS$ completely removes the bad edge and emphasizes the rest of edges, whereas the edges that connect to the bad edges have a lower value than the rest of edges (though after thresholding with our default choices they will obtain the value 1).

One can easily extend this case to higher values of $q$, $r$ and $s$, and in particular, $q=4$ and $r=s=2$. 
In this example, for any chosen $q$,  $\bS_1=\bX^q$ and thus $\bS$ assigns value 0 to the bad edge $14$. On the other hand, the values of $\bS$ to the other edges are sufficiently large. One can perform a similar analysis when $23$ is an additional bad edge and still notice that low and non-zero values of the statistics $\bS$ help in separating bad edges.

\begin{figure}[htbp]
    \centering
    \includegraphics[width=.95\columnwidth]{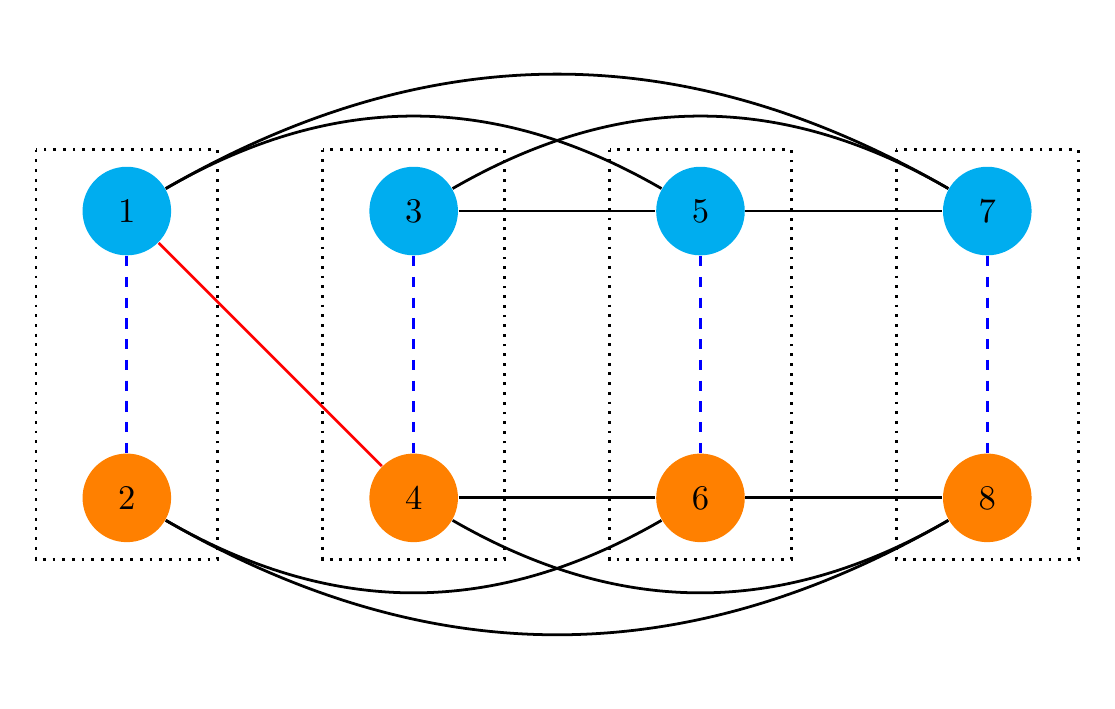}
    \includegraphics[width=.85\columnwidth]{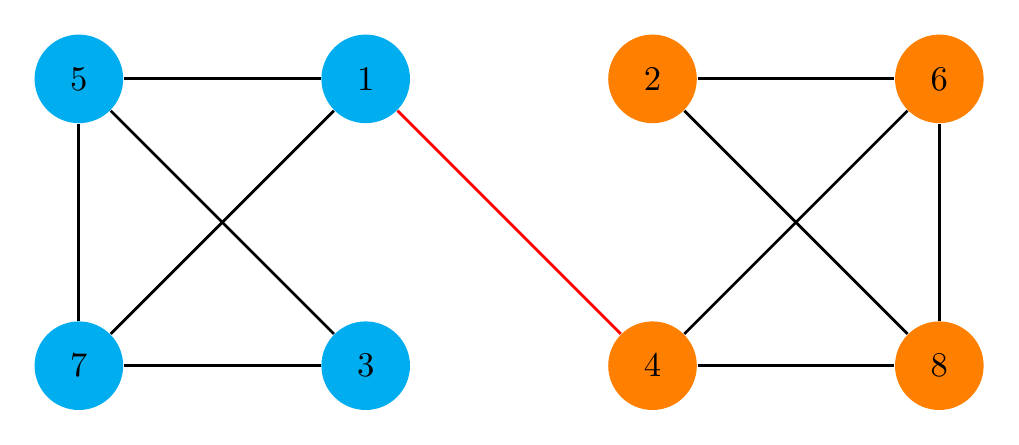}
    \caption{A specific example of keypoint matches. Top:  Boxes represent images and each node is a keypoint (clustered according to two different 3D scene points). Edges in $E_g$, $E_b$ and $E_D$ are in black, red and dashed blue, respectively. The $\bS$ statistics, computed in the text, is zero for the bad edge in $E_b$ and sufficiently large otherwise. Bottom: the view of this graph as two connected subgraphs with one bad edge in between.}
    \label{fig:ex}
\end{figure}

\end{document}